\documentclass{article}

\usepackage[preprint]{neurips_2026}

\usepackage[utf8]{inputenc} 
\usepackage[T1]{fontenc}    
\usepackage{hyperref}       
\usepackage{url}            
\usepackage{booktabs}       
\usepackage{amsfonts}       
\usepackage{nicefrac}       
\usepackage{microtype}      
\usepackage{xcolor}   
\usepackage{amsmath}
\usepackage{amsthm}
\usepackage{algorithm}
\usepackage{algorithmic}

\newtheorem{theorem}{Theorem}
\newtheorem{lemma}{Lemma}
\newtheorem{proposition}{Proposition}

\theoremstyle{definition}
\newtheorem{definition}{Definition}

\theoremstyle{remark}

\usepackage{float}
\usepackage{wrapfig}
\usepackage{subcaption}
\usepackage{graphicx}

\usepackage{bbold}
\usepackage{bbm}
\usepackage{bm}

\newcommand{\ignore}[1]{}

\newcommand{\Exp}[1]{\mathbb{E}\left[#1\right]} 

\newcommand{\E}{\mathcal{E}}
\newcommand{\N}{\mathbb{N}}

\newcommand{\remove}[1]{}

\DeclareMathOperator*{\argmax}{arg\,max}
\DeclareMathOperator*{\argmin}{arg\,min}

\title{Cost-Aware Best-LLM Identification \\ using Dueling Feedback}

\author{%
  Sarvesh Gharat \\
  Centre for Machine Intelligence and Data Science (C-MInDS)\\
  Indian Institute of Technology Bombay\\
  \texttt{sarveshgharat19@gmail.com} \\
  \And
  Nikhil Karamchandani \\
  Department of Electrical Engineering\\
  Indian Institute of Technology Bombay\\
  \texttt{nikhilk@ee.iitb.ac.in} \\
  \And
  Jayakrishnan Nair \\
  Department of Electrical Engineering\\
  Indian Institute of Technology Bombay\\
  \texttt{jayakrishnan.nair@ee.iitb.ac.in} \\
}

\begin{document}

\maketitle

\begin{abstract}
Inspired by the problem of identifying the best model from a collection of large language models (LLMs) with heterogeneous querying costs, we formulate and analyse a variant of the multi-armed bandit (MAB) with (i) dueling feedback, where pairwise comparisons between model responses provide robust preference signals, and (ii) heterogeneous sampling costs, reflecting the differing costs of querying different LLMs. Assuming the existence of a Condorcet winner, a condition we empirically validate across multiple real-world datasets, we propose a Track-and-Stop style algorithm for best-arm identification with prescribed confidence. We prove that the algorithm almost surely achieves the asymptotically optimal cost as the error tends to zero. Finally, we extensively evaluate our approach on both synthetic and real-world instances, demonstrating consistent improvements over classical cost-unaware algorithms and their cost-aware extensions.
\end{abstract}

\section{Introduction}

Sequential decision-making under uncertainty is a fundamental problem in machine learning, with the multi-armed bandit (MAB) model serving as a canonical abstraction. While classical bandits assume access to scalar rewards, many applications provide feedback only in the form of pairwise preferences. This preference-based feedback has emerged as a more reliable signal in human-centered evaluation tasks, since people often find it easier to answer “Which of these two options is better?” than to assign consistent absolute scores.
Such pairwise feedback naturally arises in diverse online learning frameworks. In information retrieval and recommender systems, for instance, users typically click or choose between competing options rather than rate them numerically \citep{yan2022human, hui2017low, roitero2022preferences}. These observations have motivated the dueling bandit problem, a variant of the MAB framework in which the learner interacts through noisy pairwise comparisons \citep{yue2009interactively, yue2012k}. 

In this work, we are motivated by the problem of identifying the best model for a certain task (say, text-to-image generation) from amongst a collection of large language models (LLMs). Preference-based feedback has received significant attention recently in LLM evaluations, as existing automatic baselines often suffer from data leakage and misalignment with human judgment \citep{chiang2024chatbot, ramos2025data, zhou2025lessleak, fodor2025line} To mitigate these issues, large-scale evaluation platforms, such as Chatbot Arena \cite{chiang2024chatbot}, have adopted head-to-head comparisons, where annotators and users provide judgments by directly comparing the outputs of two models to the same query. Formally, we model the best-LLM identification problem as a \emph{fixed confidence} dueling MAB formulation. Relative to prior literature, the novelty of this formulation is two-fold---we make the mild structural assumption of a Condorcet winner, and account for heterogeneous arm (LLM) sampling costs. We motivate each of these novelties below.

A dueling bandit instance is described via a preference matrix $\mathbb{P} =[[p_{i,j}]]$, where $p_{i,j}$ reflects the probability that arm $i$ wins over arm $j$ in a duel; see Figure~\ref{fig:vision_matrix0} for an example. 
The literature proposes several structural assumptions on this preference matrix (comprehensive surveys can be found in \cite{bengs2021preference, sui2018advancements}). This includes the \emph{low noise model} (LNM), several notions that imply a total ordering of arms (for example, \emph{strong/moderate/weak stochastic transitivity}), \emph{general identifiability}, and the existence of a \emph{Condorcet winner}. Of these, the assumption of a Condorcet winner (an arm~$i$ that beats every other with probability exceeding~$1/2,$ i.e., $p_{i,j} > 1/2$ for all $j \neq i$) is the weakest, and also the most universal in terms of applicability. Indeed, we tested several empirical preference matrices for LLM tasks, and found that the only structural assumption they \emph{all} satisfy is the existence of a Condorcet winner; see Appendix~C (all appendix references point to the appendices in the supplementary material) for an analysis of the datasets we evaluate our algorithms on, and Appendix~D for additional ones. 

\begin{wrapfigure}{r}{0.62\textwidth}
  \vspace{-10pt}
  \centering
  \includegraphics[width=0.58
  \textwidth]{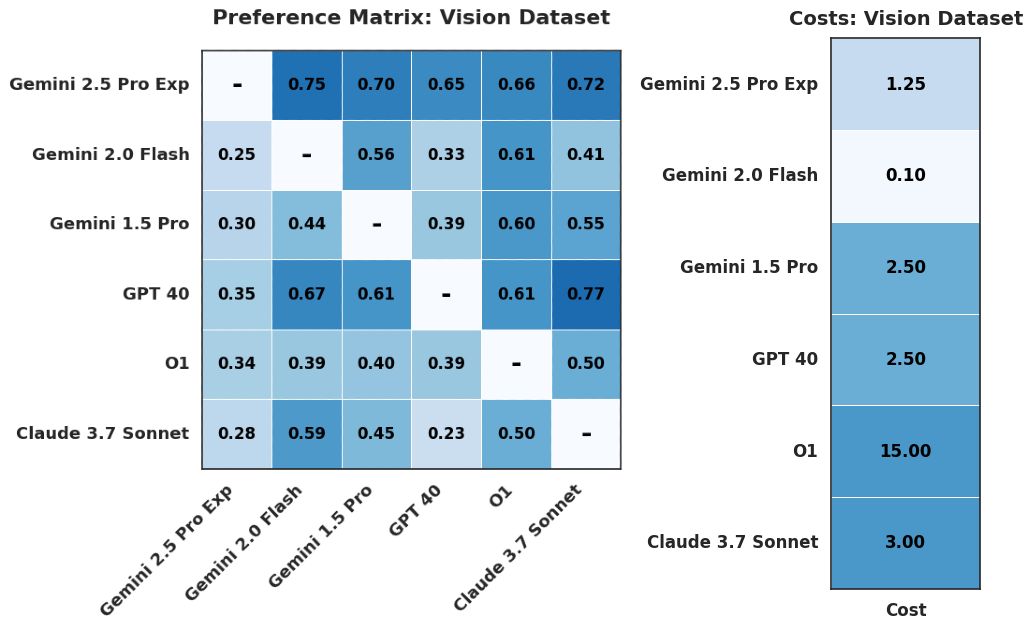}
  \caption{Preference matrix and model costs for the Vision-based task \citep{chiang2024chatbot}}
  \label{fig:vision_matrix0}
  \vspace{-10pt}
\end{wrapfigure}

An example is illustrated in Figure~\ref{fig:vision_matrix0}, which presents the preference matrix for a vision-based task from Chatbot arena \citep{chiang2024chatbot}. It can be easily checked that for this preference matrix, the Gemini 2.5 Pro Exp model is the Condorcet winner. The same preference matrix also demonstrates that transitivity (which is essential for any total ordering of LLMs) is violated---Gemini 2.0 Flash beats (with probability exceeding~$1/2$) Gemini 1.5 Pro, Gemini 1.5 Pro beats Claude 3.7 Sonnet, and Claude 3.7 Sonnet in turn beats Gemini 2.0 Flash. Such transitivity violations are quite common in real-world preference data, as we demonstrate in Appendix~C and~D.

Another practical consideration often overlooked in dueling bandit formulations is that comparisons incur heterogeneous costs. 
For instance, generating a response from a large LLM can be significantly more expensive than querying a smaller model. Figure \ref{fig:vision_matrix0} illustrates the heterogeneity in sampling/querying costs across the LLMs considered for the aforementioned vision-based task. Clearly, these diverse costs would inform the most efficient sampling strategy for inferring the best LLM. The proposed dueling MAB formulation here explicitly captures heterogeneous arm sampling costs, and our algorithms seek to minimize the average sampling cost subject to an accuracy constraint.

\noindent {\bf Related Literature: }
In this brief review, we restrict attention to the literature that addresses fixed confidence best arm identification in stochastic MABs. 
Our work is closely related to \cite{karnin2016verification}, which studies fixed-confidence best-arm identification in dueling bandits under the Condorcet-winner assumption.
However, their approach is verification-based and assumes uniform comparison costs, thereby optimizing only the number of pairwise comparisons. In contrast, we study the same identification objective under heterogeneous arm sampling costs and develop a cost-aware Track-and-Stop-style algorithm. In our experiments, this Track-and-Stop-style approach outperforms a cost-aware adaptation of their verification-based method.

While heterogeneous sampling costs have not been studied in the dueling-bandit setting, they have been considered for scalar-reward bandits for cost-aware best-arm identification~\citep{tucker2023bandits,ivkin2021cost,kanarios2024cost}. In particular, \cite{kanarios2024cost} adapts Track-and-Stop~\citep{garivier2016optimal} to heterogeneous sampling costs. However, the results/insights in these works do not extend to the dueling setting considered here.

Separately, fixed-confidence dueling bandits have also been studied under even weaker structural assumptions on the preference matrix, where the `best arm' is defined as the Borda winner \citep{busa2014pac, falahatgar2017maxing, lin2018efficient, heckel2018approximate} 
In light of our observations in the context of real-world LLM comparison datasets, we work with the (stronger) Condorcet winner assumption here. It is also worth noting that under this  stronger assumption, we are able to establish the performance guarantee of (asympotic) optimality in sampling/cost complexity relative to an information theoretic lower bound; the prior literature on identifying the Borda winner does not achieve this, to the best of our knowledge.

\noindent {\bf Summary of contributions and paper outline: }
In Section~\ref{sec: problem formulation}, we formally define the proposed dueling MAB formulation with heterogenous sampling costs. In Section~\ref{sec: lower bound}, we derive an information-theoretic lower bound on the expected cost of any algorithm that meets the accuracy constraint. We further characterize the associated `optimal' sampling proportions and use this characterization to obtain a closed-form expression for the lower bound. In Section~\ref{sec: algorithm}, we propose a cost-aware Track-and-Stop style algorithm for the dueling bandit setting. The main technical challenge is that the optimal allocation weights need not be unique, making it difficult to directly apply standard Track-and-Stop arguments that rely on unique sampling proportions~\citep{garivier2016optimal}. While similar issues have been studied in stochastic bandits~\citep{reddy2023best, al2021navigating}, the pairwise preference structure and heterogeneous costs require additional care. We also modify the GLR-based stopping rule to test the best-arm hypothesis rather than standard pairwise superiority, which is crucial for proving correctness and asymptotic optimality in terms of cost complexity. Finally, in Section~\ref{sec: numerics}, we present experiments on synthetic and real-world datasets and compare our algorithm with relevant baselines.
\section{Problem Formulation}
\label{sec: problem formulation}

We consider an $L$-armed dueling multi-armed bandit problem, where a learner is given access to $L$ arms, denoted by $\mathcal{L}:= [L]$.\footnote{For $n \in \N,$ the set $[n]$ denotes $\{1,2,\ldots,n\}.$} Unique to this work, we assume that each arm $i \in \mathcal{L}$ has an associated sampling cost~$c_i.$ The vector of costs $\mathcal{C}:= (c_1, c_2, \dots, c_L)$ is known a priori to the learner. At each time step $t \in \{1, 2, \dots \}$, the learner selects a pair of arms $a_t:= \left(a_{t,1}, a_{t,2} \right)$ where $a_{t,1}, a_{t,2} \in \mathcal{L},$ incurs a cost $c_{a_{t,1}} + c_{a_{t,2}},$ and observes a dueling outcome $y_t \in \{a_{t,1},a_{t,2}\},$ representing the \emph{winner} of the duel between the sampled arms. The outcome of dueling between any pair of arms~$i$ and~$j$ at each time step is an independent Bernoulli random variable, such that $i$ wins the duel with probability $p_{i,j}$. The winning probabilities can be compactly represented as an $L \times L$ preference matrix $\mathbb{P} =[[p_{i,j}]].$ 
Note that the diagonal entries of $\mathbb{P}$ are inconsequential, and that this matrix is anti-symmetric, meaning that for every $i,j \in [L],$ $i \neq j,$ we have $p_{j,i} = 1-p_{i,j}$. To summarize, an instance of the dueling MAB with heterogeneous sampling costs is given by the tuple $\left(\mathcal{C},\mathbb{P}\right).$ 

Furthermore, we assume the existence of a \emph{Condorcet winner}--an arm that beats every other arm with a probability greater than $0.5$. The Condorcet winner is also referred to as the \emph{best arm}. The best arm corresponding to the instance $\nu = (\mathcal{C},\mathbb{P})$ is denoted by ~$a^*(\nu);$ note that $p_{a^*(\nu),j} > 0.5, \; \forall j \neq a^*(\nu).$

Since the learner does not know the preference matrix $\mathbb{P}$ a priori, it must estimate the best arm (i.e., the Condorcet winner) by sequential sampling of arm pairs. The goal of the learner is to identify the best arm with a prescribed confidence level, while incurring the least aggregate sampling cost.\footnote{If the sampling costs are identical, note that this reduces to the standard fixed confidence dueling MAB formulation (see, for example, \cite{bengs2021preference}), where the goal is to identify the best arm with a prescribed confidence level, using the least number of duels.} Formally, after $\tau$ time steps,
the learning algorithm can either choose to continue sampling, or stop and return its estimated best arm $\hat{a}_\tau.$ Given a prescribed error threshold $\delta \in \left(0,1\right)$, we say that an algorithm is $\delta$-probably correct ($\delta$-PC) if, for any problem instance $\nu = \left(\mathcal{C},\mathbb{P}\right)$, $$P(\tau_{\delta} < \infty,\hat{a}_{\tau_{\delta}} \neq a^*(\nu)) \leq \delta,$$ where $\tau_{\delta}$ denotes the stopping time of the algorithm. Our goal is to design $\delta$-PC algorithms that incur the minimum sampling cost~$$J(\tau_\delta) = \sum_{t = 1}^{\tau_\delta} (c_{a_{t,1}} + c_{a_{t,2}}).$$ 
More precisely, as is typical in the fixed confidence setting, we seek to minimize $\Exp{J(\tau_\delta)},$ i.e., the expected cost incurred in identifying the best arm.

\newcommand{\K}{\mathcal{K}}
\label{sec: lower bound}

In this section, we establish an information-theoretic lower bound on the
\emph{expected} sampling cost incurred by any $\delta$-PC algorithm.
Interestingly, we are able to express the lower bound in an explicit and interpretable closed form—this is enabled by our modeling assumption that a Condorcet winner exists.

Let $\mathcal{K}$ represent the set of arm pairs; formally, $$\K = \{(i,j)\ \colon\ 1 \leq i < j \leq L\}.$$
Let $\E$ denote the set of instances that admit a Condorcet winner. For~$\nu = (\mathcal{C},\mathbb{P}) \in \E,$ let $\epsilon_{\text{alt}} \left(\nu\right)$ denote the set of instances in $\E$ (i.e., instances that also admit a Condorcet winner) having the same cost vector, where the best arm differs from $a^*(\nu),$ i.e.,  
\begin{align*}
\epsilon_{\text{alt}} \left(\nu\right) = \{ \nu'= (\mathcal{C},\mathbb{P}') \in \E \ \colon \ a^*\left(\nu'\right) \neq a^*\left(\nu\right)\}.
\end{align*}
The following lemma, which is a straightforward generalization of Theorem 1 in \cite{garivier2016optimal}, establishes a lower bound on the expected sampling cost of any $\delta$-PC algorithm. 
\begin{lemma}
\label{lemma:lower_bound_abstract}
For $\nu \in \E,$ 
let $J\left(\tau_{\delta}\right)$ denote the total sampling cost incurred by a $\delta$-PC algorithm. Then $E_{\nu}[J\left(\tau_{\delta}\right)] \geq c^*\left(\nu\right)\log\left(\frac{1}{4\delta}\right)$ where
    \begin{align}
    \label{eq:lower_bound_generic}
    c^*\left(\nu\right)^{-1} = \sup_{\mathbf{w} \in \Sigma_{\K}} \ \inf_{\nu' = (\mathcal{C},\mathbb{P}') \in \epsilon_{\text{alt}} \left(\nu\right)}\ 
    \left( \sum_{(i,j) \in \K} \dfrac{w_{i,j}}{c_i + c_j} d\left(p_{i,j}, p'_{i,j} \right) \right).
    \end{align}
    Here, $\Sigma_{\K}$ denotes the probability simplex over~$\K,$ and $d(x,y)$ denotes the KL divergence between Bernoulli distributions with means~$x$ and~$y.$
\end{lemma}
Lower bounds of this type are standard in the fixed confidence MAB literature \citep{garivier2016optimal, poiani2024best,jedra2020optimal, reddy2023best}. While such lower bounds tend to be (asymptotically) tight, they are not particularly interpretable, as they are expressed in terms of the solution of a certain nested optimization problem. Moreover, the set $\mathcal{W}^*\left(\nu\right)$ of probability distributions over~$\K$ that attain the outer supremum in \eqref{eq:lower_bound_generic} is significant from the standpoint of an optimal algorithm--it contains the `ideal' fractional cost allocations (across arm pairs) under an optimal algorithm; see~\cite{garivier2016optimal,kanarios2024cost}. These ideal allocations also do not typically admit a closed-form characterization, and must be computed numerically. However, given the specific structure of the present formulation (specifically, the assumption that the problem instances under consideration admit a Condorcet winner), we show that the lower bound in Lemma~\ref{lemma:lower_bound_abstract} can be expressed in closed form. We are also able to characterize $\mathcal{W}^*\left(\nu\right)$ in explicit closed form (see Proposition~\ref{prop: optimal-weights} below); this characterization is exploited by the proposed algorithm (see Section~\ref{sec: algorithm}).

\begin{theorem}
\label{theorem: minmax}
Given an error threshold~$\delta \in (0,1),$ let $J\left(\tau_{\delta}\right)$ denote the cost incurred by a $\delta$-PC algorithm. Then
\begin{align*}
\dfrac{E_{\nu}[J\left(\tau_{\delta}\right)]}{\log\left(1/4\delta \right)} \geq \sum_{m \in \mathcal{L}\setminus\{a^*(\nu)\}} \dfrac{1}{\max_{n \neq m} \left( \dfrac{1}{c_m + c_n} d\left(p_{m,n}, \max(0.5, p_{m,n})\right) \right)}.
\end{align*}
\end{theorem}

While the proof of Theorem~\ref{theorem: minmax} is deferred to
Appendix~\ref{app: lower bound}, we state below a consequence of the proof
that will be useful in designing the sampling rule. Define, for $m \neq a^*(\nu),$
\begin{align*}
\alpha_m &:= \max_{n \neq m} \left( \dfrac{1}{c_m + c_n} d\left(p_{m,n}, \max(0.5, p_{m,n})\right) \right),\\
\Gamma_m &:= \argmax_{n \neq m} \left( \dfrac{1}{c_m + c_n} d\left(p_{m,n}, \max(0.5, p_{m,n})\right) \right).
\end{align*}
Intuitively, $\alpha_m > 0$ denotes the statistical ease (in terms of cost) of identifying arm~$m$ as non-optimal; a large value of~$\alpha_m$ indicates that is `cheap' to infer that is not optimal. The set~$\Gamma_m$ represents the set of arms (that all beat arm~$m$ with probability exceeding~$1/2$) that are most efficient (in terms of cost) at making this inference via pairwise comparisons. 


\begin{proposition}
\label{prop: optimal-weights}
For $w^* \in \mathcal{W}^*\left(\nu\right),$ define~$w^*_m = \sum_{(k,\ell) \in \mathcal{K} \colon m \in \{k,\ell\}} w_{k,\ell}$ as the total probability assigned to comparisons involving arm~$m.$ Then for~$m \neq a^*(\nu),$ 
\begin{equation*}
   w^*_m = \frac{1}{\alpha_m} \frac{1}{\sum_{n \neq a^*(\nu)} \frac{1}{\alpha_n}}. 
\end{equation*}
Moreover, $w^*$ concentrates~$w^*_m$ entirely (and arbitrarily) over comparisons with arms in~$\Gamma_m.$
\ignore{
For each arm $i \in \mathcal{L}\setminus\{a^*(\nu)\},$ let
$w_i := \sum_{j\neq i} w_{i,j}$ denote the total weight assigned to arm
$i.$ Then the optimal choice of these weights is given by
\begin{align}
\label{eq: optimal-w-i}
w_i
=
1 \Bigg/
\sum_{m \in \mathcal{L}\setminus\{a^*(\nu)\}}
\dfrac{
\max_{j \neq i}
\left(
\dfrac{1}{c_i+c_j}
d\left(p_{i,j},\max(0.5,p_{i,j})\right)
\right)
}{
\max_{n \neq m}
\left(
\dfrac{1}{c_m+c_n}
d\left(p_{m,n},\max(0.5,p_{m,n})\right)
\right)
}.
\end{align}
Moreover, for each $i,$ the weight $w_i$ is assigned only to those pairs
$(i,j)$ that attain the maximum in~\eqref{eq: optimal-w-i}. If multiple
pairs attain this maximum, then $w_i$ can be split arbitrarily among
them.
}
\end{proposition}
Note that Proposition~\ref{prop: optimal-weights} provides a complete characterization of~$\mathcal{W}^*\left(\nu\right).$ The distributions in this set assign, to comparisons involving each sub-optimal arm~$m,$ a probability that is inversely proportional to $\alpha_m.$ Assuming that the set~$\Gamma_m$ consists of a unique arm for each~$m \neq a^*(\nu),$ there is a unique optimal distribution in $\mathcal{W}^*\left(\nu\right),$ which is supported over exactly~$L-1$ pairs, each pair representing the most statistically efficient way of rejecting one sub-optimal arm. On the other hand, if $\Gamma_m$ is not a singleton for any sub-optimal arm~$m$, there are infinitely many optimal cost allocation distributions, since the probability $w^*_m$ can be distributed arbitrarily over~$|\Gamma_m|$ different pairs.

In the following section, we present an algorithm for identifying the Condorcet
winner that asymptotically, as $\delta \downarrow 0$, matches the cost lower
bound in Theorem~\ref{theorem: minmax}. 
central role in the design and analysis of the algorithm.
\section{Algorithm}
\label{sec: algorithm}

In this section, we present the Dueling Bandit Cost-Aware Track and Stop (DCTAS) algorithm, which identifies the optimal arm subject to a $\delta$-PC guarantee, while minimizing the cost of comparisons. The algorithm leverages a Chernoff Stopping Rule to determine when sufficient confidence has been achieved in the selection process. Additionally, it employs a cost-aware exploration strategy, ensuring that comparisons between the arms are allocated optimally to reduce the sampling cost. The algorithm is described formally as Algorithm~\ref{algo: DCTAS}. Here, for $k \in \K,$ $N_k(t)$ denotes the number of times arm pair~$k$ has been sampled prior to time~$t.$ $\hat{\mathbb{P}}(t) = [[\hat{p}_{i,j}(t)]]$ denotes the empirical estimate of the preference matrix using the observations prior to time~$t.$ The \textit{WeightPullAllocation} routine is described in Section~\ref{sec:algo_sampling}. Finally, the likelihood ratio statistics~$Z_{i,j}(t)$ and the stopping threshold~$\beta(t,\delta)$ are defined in Section~\ref{sec:algo_stopping}.

\begin{algorithm}[ht]
\caption{Dueling Bandit Cost Aware Track and Stop (DCTAS)}
\label{algo: DCTAS}
\begin{algorithmic}
  \STATE \textbf{Input:} Error Threshold $\delta$; 
  Cost Vector $\mathcal{C}$
  \STATE \textbf{Output:} Best arm $i \in \mathcal{L}$
\end{algorithmic}
\begin{algorithmic}[1]
\STATE Pull each arm pairs $i, j \in \mathcal{L}$ once for initialization;
\FOR{$t=1,2, \dots$}
\STATE $\alpha\left(t\right) = \text{WeightPullAllocation}\left(\mathcal{C}, \hat{\mathbb{P}}\left(t\right) \right)$;
\IF{$\exists k \in \mathcal{K} \mid N_{k}\left(t\right) < \sqrt{t}$}
\STATE $a_t = \argmin_{k \in \mathcal{K}} N_k\left(t\right)$;
\ELSE
\STATE $a_t = \argmin_{k \in \mathcal{K}} N_k\left(t\right) - \sum_{s=1}^t \alpha_{k}\left(s\right)$ ;
\ENDIF
\STATE Pull pair $a_t$
\IF{$\exists i \in \mathcal{L}$ such that $Z_{i,j}\left(t\right) > \beta\left(t,\delta\right)$ for all $j \in \mathcal{L} \setminus \{i\},$}
\RETURN $i$
\ENDIF
\ENDFOR
\end{algorithmic}
\end{algorithm}

\subsection{Sampling Rule}
\label{sec:algo_sampling}
The sampling rule in DCTAS is designed to efficiently allocate comparisons among arms while balancing exploration and cost minimization. At each time step, the algorithm must determine which pair of arms to sample next, leveraging both past observations and a cost-aware allocation strategy to guide this selection. 

The algorithm begins by sampling each pair of arms once, establishing an initial estimate of their preference probabilities. Subsequently, it follows an adaptive sampling strategy: if any arm pair has been sampled fewer than $\sqrt{t}$ times, it is selected for comparison to ensure sufficient exploration. Otherwise, the algorithm prioritizes the pair that minimizes the gap between its observed sample count and its time average target allocation. This target allocation intuitively represents the fraction of times a pair of arms is to be pulled by time~$t$. It is important to note that this value is just an empirical estimate that converges to the set of optimal allocations with the increase in the number of observations.
\begin{algorithm}[ht]
\caption{WeightPullAllocation}
\label{algo: WA}
\begin{algorithmic}
  \STATE \textbf{Input:} Cost Vector $\mathcal{C}$; Preference Matrix $\mathbb{P}$
  \STATE \textbf{Output:} Set of Optimal Pull Fractions $\alpha$
\end{algorithmic}
\begin{algorithmic}[1]
\STATE Calculate $w_i = w^*(\hat{\nu})\; \forall i \in \mathcal{L}$ using Proposition \ref{prop: optimal-weights}; 
\FOR{$i \in \mathcal{L}$}
\STATE $z \leftarrow \sum_{j\neq i} \mathbb{1}\left(\frac{ d\left(p_{i,j}, \max{(0.5, p_{i,j})}\right)}{c_i + c_{j}} = \max_{j \neq i}\frac{ d\left(p_{i,j}, \max{(0.5, p_{i,j})}\right)}{c_i + c_{j}}\right)$;
\FOR{$j \in \mathcal{L} \backslash i$}
\STATE $w_{i,j} \leftarrow \frac{w_i}{z}\mathbb{1}\left( \frac{ d\left(p_{i,j}, \max{(0.5, p_{i,j})}\right)}{c_i + c_{j}} = \max_{j \neq i}\frac{ d\left(p_{i,j}, \max{(0.5, p_{i,j})}\right)}{c_i + c_{j}} \right)$;
\ENDFOR
\ENDFOR
\STATE $\alpha \leftarrow \frac{w/\mathcal{C}}{\sum_{(m,n)\in\mathcal{K}} w_{m,n}/\left(c_m + c_n\right)}$ ;
\RETURN $\alpha$
\end{algorithmic}
\end{algorithm}
To compute this estimated target allocation at any time $t$, we first plug in the empirical estimates of the preference matrix along with the associated arm costs into \textit{WeightPullAllocation}, as defined in Algorithm \ref{algo: WA}. Intuitively, this routing transforms a cost-based allocation (obtained by applying Proposition~\ref{prop: optimal-weights} to the empirical instance estimate) to a pull-based allocation. 
\begin{lemma}
\label{lemma: mssr5eq1}
    Let $\mathcal{W}^*\left(\nu\right)$ denote the set of optimal cost fractions, and let $\alpha^*\left(\nu\right)$ denote the set of optimal pull fractions. For each $\mathcal{W}\left(\nu\right) \in \mathcal{W}^*\left(\nu\right)$ there exists a unique corresponding $\alpha\left(\nu\right) \in \alpha^*\left(\nu\right)$, and vice versa. That is, $\exists$ a bijective function $f:(\mathcal{W}^*\left(\nu\right), \mathcal{C}) \mapsto \alpha^*\left(\nu\right)$ such that for all $\mathcal{W}\left(\nu\right) \in \mathcal{W}^*\left(\nu\right)$, we have $\alpha\left(\nu\right) = f(\mathcal{W}\left(\nu\right), \mathcal{C})$ and for each arm pair $\left(l_i,l_j\right) \in \mathcal{L}$, the optimal pull fraction satisfies
    \begin{align*}
        \alpha_{i,j} = f_{i,j}\left(\mathcal{W}, \mathcal{C}\right) = \dfrac{\mathcal{W}_{i,j}/\left(c_i+c_j\right)}{\sum_{(m,n)\in\mathcal{K}} \mathcal{W}_{m,n}/\left(c_m + c_n\right)}
    \end{align*}
\end{lemma}
This explicit conversion from cost fraction to pull fraction is crucial in establishing the asymptotic cost complexity of DCTAS. Having established the sampling rule, we next describe the stopping rule that determines when DCTAS terminates with high confidence.

\subsection{Stopping and Decision Rule}
\label{sec:algo_stopping}
We utilize the standard Generalized Likelihood Ratio Test (GLRT)-based stopping rule to determine when the algorithm has collected a sufficient number of samples to identify the optimal arm with enough confidence. Specifically, at each time step $t$, we compute the GLR statistic that measures the likelihood of the current best arm being optimal relative to alternative candidates. Mathematically, it is represented by
\begin{align*}
Z_{i,j}(t)
&= \log\!\left(
\frac{\displaystyle \max_{\text{$i$ is best}} \prod_{m=1}^{L-1}\prod_{n=m+1}^{L}\lambda_{m,n}(t)}
     {\displaystyle \max_{\text{$j$ is best}} \prod_{m=1}^{L-1}\prod_{n=m+1}^{L}\lambda_{m,n}(t)}
\right) \\
\intertext{where}
\lambda_{m,n}(t)
&= \bigl(p_{m,n}\bigr)^{N_{m,n}(t)\,\hat p_{m,n}(t)}
   \bigl(1-p_{m,n}\bigr)^{N_{m,n}(t)\,\bigl(1-\hat p_{m,n}(t)\bigr)} 
\end{align*}
and $Z_{i,j}$ is the GLR statistic between arm $i$ and $j$. This statistic has a closed-form solution, as presented in Lemma \ref{lemma: mssr1eq1}.
\begin{lemma}
\label{lemma: mssr1eq1}
At any time $t$, for every arm pair $i, j$, the GLR statistic $Z_{i,j}\left(t\right)$ follows the following closed-form solution
\begin{align*}
Z_{i,j}(t)
= \sum_{k \in \mathcal{L}} N_{j,k}(t)\,
   d\!\left(\hat p_{j,k}(t),\, \max\{\hat p_{j,k}(t),\, 0.5\}\right) - \sum_{k \in \mathcal{L}} N_{i,k}(t)\,
   d\!\left(\hat p_{i,k}(t),\, \max\{\hat p_{i,k}(t),\, 0.5\}\right).
\end{align*}
and for every $i, j$ we have $Z_{i,j}\left(t\right) = - Z_{j,i}\left(t\right)$
\end{lemma}
Now, let $Z\left(t\right):= \max_{i \in \mathcal{L}} \min_{j \in \mathcal{L}\backslash i} Z_{i,j}\left(t\right)$, then hypothesis testing intuition suggests the following stopping rule.
\begin{align*}
\tau_\delta
 &= \inf\{\,t\in\mathbb{N} : Z(t) > \beta(t,\delta)\,\},\\
\intertext{where}
\beta(t,\delta)
 &= 3(L-1)\,\log\!\left(1+\log\frac{t}{L-1}\right)
  + (L-1)\,C\!\left(\frac{\log\frac{L-1}{\delta}}{L-1}\right).
\end{align*}
Here, $\tau_{\delta}$ denotes the stopping time, and $C\left(x\right)$ is defined in Definition \ref{def: C_exp}.
The threshold $\beta(t,\delta)$ acts as a dynamic confidence boundary that grows logarithmically with time, ensuring that premature termination is avoided while still guaranteeing $\delta$-PC correctness, and is a natural adaptation of the threshold provided by \cite{kaufmann2021mixture}. 
\begin{definition}
\label{def: C_exp}
    Let $h\left(u\right) = u - \log u$, and $h^{-1}\left(u\right)$ be its inverse. And let $\tilde{h}_z\left(x\right)$ be defined as
    \begin{equation*}
    \tilde{h}_z\left(x\right) =
    \begin{cases}
      e^{1/h^{-1}\left(x\right)} h^{-1}\left(x\right), & \text{if}\  x>h(1/\log z)\\
      z\left(x-\log\log z \right), & \text{otherwise}
    \end{cases}
    \end{equation*}
Then $C(x)$ is defined by 
\begin{align*}
    C\left(x\right) &= 2\tilde{h}_{3/2}\left(\dfrac{h^{-1}\left(1+x\right) + \log\left(\pi^2/3 \right)}{2}\right)
\end{align*}
\end{definition}
As the Chernoff Statistic $\left(Z\left(t\right)\right)$ measures the distance between an instance where the current empirical best arm is indeed the optimal arm and the closest instance where it is not, the larger value of $Z\left(t\right)$ indicates more confidence in the prediction. As a result, the stopping rule ensures that the algorithm terminates only when it is sufficiently confident about its prediction. This guarantee is formally established in the following theorem.
\begin{theorem}
    \label{theorem: delta-pac}
    Let $\delta \in \left(0,1\right)$. Then, the DCTAS algorithm presented in Algorithm \ref{algo: DCTAS} is $\delta$-PC i.e.,
    \begin{align*}
        P_{\nu}\left(\tau_\delta < \infty, \hat{a}_{\tau_\delta} \neq a^*\left(\nu\right)\right) \leq \delta
    \end{align*}
\end{theorem}
We conclude this section by providing the almost-sure guarantees on the cost complexity of the DCTAS algorithm as stated in Theorem \ref{theorem: sample complexity}. 
\begin{theorem}
\label{theorem: sample complexity}
For all $\delta \in (0,1)$, the DCTAS algorithm, as described in Algorithm \ref{algo: DCTAS}, almost surely satisfies
    \begin{align*}
        \limsup_{\delta \rightarrow 0} \dfrac{J\left(\tau_{\delta}\right)}{\log\left(1/\delta\right)} &\leq c^*\left(\nu\right)  
    \end{align*}
    where $c^*\left(\nu\right)$ is an instance dependent constant as defined in Lemma \ref{lemma:lower_bound_abstract}
\end{theorem}
The above result establishes an almost-sure upper bound on the total cost incurred by the DCTAS algorithm while maintaining the $\delta$-PC guarantee. Furthermore, this upper bound asymptotically matches the corresponding lower bound as $\delta \to 0$, thereby establishing the asymptotic optimality of DCTAS in terms of cost complexity.

\section{Numerical Experiments}
\label{sec: numerics}

In this section, we evaluate the empirical performance of DCTAS on both synthetic and real-world dueling bandit instances. We compare DCTAS against four baselines: TAS, the cost-unaware version of our method; CRR and DPCA, two cost-aware baselines; and DCTAC, which uses the same tracking-based sampling rule as DCTAS but a confidence-interval-based stopping rule. Detailed descriptions, pseudocodes, and correctness arguments for all baselines are provided in Appendix~\ref{app: baselines}. We first present a synthetic example illustrating the impact of costs on the optimal sampling strategy in Section~\ref{sec:synthetic}, followed by experiments on real-world LLM comparison datasets in Section~\ref{sec:realworld}. Since these experiments use fixed preference matrices and sampling costs derived from existing datasets, they do not require model training or GPU inference, and were run on a Google Cloud \texttt{e2-highcpu-8} instance with 8 vCPUs and 8~GB of memory.

\subsection{Synthetic instance}
\label{sec:synthetic}
We consider a small synthetic dueling bandit instance with three arms, where the pairwise preference probability matrix $\mathbb{P}$ is given by 
\begin{align}
\label{dataset: preference matrix}
\mathbb{P} = \begin{bmatrix}
- & 0.63 & 0.65 \\
0.37 & -  & 0.6  \\
0.35 & 0.4  & - 
\end{bmatrix}
\end{align}
In this instance, Arm 1 is the Condorcet winner, as it outperforms both Arms 2 and 3 in pairwise comparisons. To study the impact of sampling costs on the performance of various algorithms, we assign a cost vector $\mathcal{C} = (k, 1, 1)$, where the parameter $k$ is varied between $1$ and $19$. While Arms $2$ and $3$ have a fixed querying cost of $1$, the cost of sampling Arm $1$ is increased to observe how the algorithms adapt their sampling strategy to minimize the \textit{total cost} required to identify the best arm. For each value of $k$, the experiment is executed over $500$ independent trials. 
\begin{wrapfigure}{r}{0.47\textwidth}
    \vspace{-10pt}
    \centering
    \includegraphics[width=0.47\textwidth]{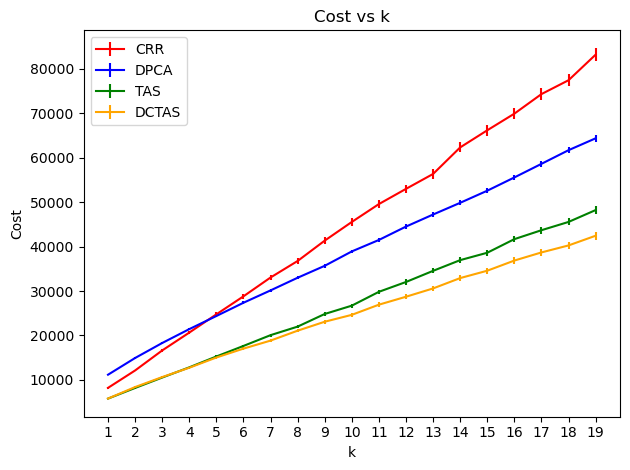}
    \caption{Total cost incurred by DCTAS, TAS, CRR, and DPCA for different values of $k$.}
    \label{fig:costvsk}
    \vspace{-10pt}
\end{wrapfigure}
In Figure \ref{fig:costvsk}, we plot the average total cost (along with the $95\%$ confidence interval) incurred to identify the Condorcet winner with a fixed confidence level, $\delta=0.01$, for different values of $k$. 
As the figure illustrates, the DCTAS algorithm consistently identifies the Condorcet winner with a lower total cost across all values of $k$. While CRR and DPCA are both cost-aware algorithms, they do not perform as well as DCTAS, which by its very design, tunes its sampling proportions optimally in response to the different arm costs. 

We also observe that, for smaller values of $k$, the mean cost incurred by TAS and DCTAS is almost the same. This is because when $k$ is small, i.e., $k<4$, the optimal sampling strategy involves querying the pairs $(1, 2)$ and $(1, 3)$ since the cost of querying Arm $1$ is comparable to that of the other arms. In this regime, both TAS and DCTAS follow similar querying patterns, resulting in similar total costs. However, as the value of $k$ increases, the cost of querying Arm $1$ becomes significantly higher. In response, DCTAS adapts its sampling strategy by prioritizing less expensive comparisons, primarily focusing on the pairs $(1, 2)$ and $(2, 3)$ (refer to Table \ref{tab: kvscost}). By leveraging the transitivity in preferences and reducing reliance on costly comparisons involving Arm $1$, DCTAS can minimize the total querying cost while still reliably identifying the best arm. This adaptive behavior leads to a substantial reduction in total cost compared to TAS as $k$ increases.

\subsection{Real-world datasets}
\label{sec:realworld}

We further evaluate our algorithms on real-world datasets released by Chatbot Arena \cite{chiang2024chatbot}. These datasets provide pairwise performance data for different LLMs across text-to-image generation, text-to-text generation, vision-based tasks, and search-based tasks. We use this preference data, together with the corresponding model query costs, to construct dueling bandit instances. As mentioned in the introduction, all considered datasets admit a Condorcet winner. A brief description of each dataset and further details of the experiment design are provided in the Appendix \ref{App: DD}.

\begin{wraptable}{r}{0.45\textwidth}
\vspace{-8pt}
\caption{Optimal pull fractions for cost-aware and cost-unaware algorithms for $k=4$.}
\label{tab: kvscost}
\centering
\begin{tabular}{|l|c|c|} 
\hline
Arm Pairs & Cost Unaware & Cost Aware \\ 
\hline\hline
1 and 2 & $0.5720$ & $0.3706$ \\ 
\hline
1 and 3 & $0.4280$ & $0$ \\
\hline
2 and 3 & $0$ & $0.6294$ \\
\hline
\end{tabular}
\vspace{-8pt}
\end{wraptable}

For each of the above identified dueling bandit instances, we compare the performance of our proposed algorithm with the various baselines. We set the confidence parameter to $\delta = 10^{-10}$ and execute 100 independent trials for each datapoint. We report the mean cost incurred by each algorithm, along with 99.7\% confidence intervals, in Table \ref{tab: realworld_results}. The missing entries in the table correspond to (instance, scheme) pairs for which the cost was prohibitively large; in particular the average cost for these cases would be larger than ten times the average cost incurred by the DCTAS algorithm. 

\begin{table}[htb]
\caption{Performance of Algorithms across Real-World Datasets.}
\label{tab: realworld_results}
\begin{center}
\begin{tabular}{| l | c | c | c | c |} 
\hline
 & T2I & T2T & Vision & Search \\ 
\hline\hline
DCTAS & $2823 \pm 89$ & $21.5 \pm 0.9$ & $77.0 \pm 4.3$ & $1015 \pm 28$ \\
\hline
TAS & $2931 \pm 171$& $22.8 \pm 1.2$ & $78.9 \pm 4.0$ & $1079 \pm 59$ \\
\hline
DPCA & $13017 \pm 138$ & $-$ & $304.2\pm 70.4$ & $9184\pm 212$ \\
\hline
CRR & $-$ & $73.8 \pm 2.0$ & $364.2 \pm 16.3$ & $-$ \\ \hline
DCTAC & $1461 \pm 907$ & $12.8 \pm 1.7$ & $42.6 \pm 2.9$ & $973 \pm 67$\\
\hline
\end{tabular}
\end{center}
\end{table}

As shown in Table~\ref{tab: realworld_results}, the DCTAC algorithm incurs the lowest cost across all evaluated datasets. Recall that DCTAC has the sample tracking-based sampling rule as DCTAS; its advantage arises from the confidence interval-based stopping rule it uses, which allows for earlier termination compared to the Chernoff-based rule used by DCTAS. On the other hand, among the other baselines, we observe that the Track-and-Stop type algorithms outperform the Cost-Aware Round-Reject (CRR) and Dual-Phase Cost-Aware (DPCA) algorithms. This can be attributed to the adaptive sampling strategy used by these algorithms, in contrast to the more static approaches of CRR and DPCA. 

\begin{wraptable}{r}{0.48\textwidth}
\vspace{-8pt}
\caption{Percentage reduction in average cost of DCTAS relative to TAS across real-world datasets.}
\label{tab: dctasvstas}
\centering
\small
\begin{tabular}{lcccc} 
\hline
 & T2I & T2T & Vision & Search \\ 
\hline
DCTAS & $2823$ & $21.5$ & $77.0$ & $1015$ \\
TAS & $2931$ & $22.8$ & $78.9$ & $1079$ \\
$\%$ reduction & $3.8\%$ & $5.7\%$ & $2.4\%$ & $5.9\%$ \\
\hline
\end{tabular}
\vspace{-8pt}
\end{wraptable}

To further understand the impact of cost-aware sampling, we compare the DCTAS and TAS algorithms, both of which employ the Chernoff-based stopping rule but differ in their sampling strategies. While DCTAS adaptively allocates samples to minimize the total cost, the TAS algorithm ignores cost and allocates samples to minimize the number of pulls.  Table~\ref{tab: dctasvstas} reports the percentage reduction in the mean cost achieved by DCTAS relative to TAS across all datasets. We note that DCTAS achieves close to $6\%$ lower cost on the T2T and Search datasets. We find that these datasets have a higher degree of cost heterogeneity, emphasizing the importance of adaptive cost-aware sampling in realistic evaluation settings.



\section{Conclusion}

This paper addresses the challenge of selecting the most suitable model from a pool of large language models (LLMs) that vary in both quality and querying cost. We frame the problem as a novel extension of the dueling multi-armed bandit setting that integrates non-uniform sampling costs, which account for the practical differences in accessing each model. Assuming the presence of a Condorcet winner—a property we verify empirically using several real-world datasets—we develop a Track-and-Stop style algorithm tailored for cost-sensitive best-arm identification under dueling feedback. We show that this approach converges to the optimal cost-efficiency in the limit of vanishing error probability. Our experimental results on synthetic and real-world benchmarks highlight the algorithm’s consistent advantages over both standard and cost-aware baselines.


\newpage
\bibliographystyle{unsrtnat}
\bibliography{references}


\newpage
\appendix

\appendix
\section{Missing Proofs of Lemmas and Theorems}
\label{app: missing proofs}

In this section, we provide complete proofs of the results stated in the main text. We begin in Appendix~\ref{app: lower bound} with the derivation of the lower bound.


\subsection{Lower Bound}
\label{app: lower bound}

\begin{proof}[Proof of Lemma \ref{lemma:lower_bound_abstract}]
The result is straightforward when $E_{\nu}[J(\tau_{\delta})] = \infty$. Hence, for the rest of the proof, we assume that $E_{\nu}[J(\tau_{\delta})] < \infty$, which implies $E_{\nu}[N(\tau_{\delta})] < \infty$, where $E_{\nu}[N(\tau_{\delta})]$ is the expected number of total pulls made by any $\delta$-PC algorithm. Consequently, we have $P_{\nu}(\tau_\delta = \infty) = 0$.

Let $\nu' \in \epsilon_{\text{alt}}(\nu)$, and let $E$ be the event defined as $E = \{\tau_{\delta} < \infty \text{ and } o \notin a^*(\nu')\}$, where $o \in \mathcal{L}$ is the output arm recommended at the stopping time. Then,
\begin{align}
    2\delta &\geq P_{\nu}\left(\tau_\delta < \infty \text{ and } o \notin a^*(\nu)\right) + P_{\nu'}\left(\tau_\delta < \infty \text{ and } o \notin a^*(\nu')\right) \label{lb1eq1} \\
    &\geq P_{\nu}(E^c) + P_{\nu'}(E) \label{lb1eq2} \\
    &\geq \frac{1}{2} \exp\left( -\sum_{(i,j)\in\mathcal{K}} E_{\nu}[N_{i,j}(\tau_{\delta})]\, d(p_{i,j}, p'_{i,j}) \right). \label{lb1eq3}
\end{align}
Rearranging \eqref{lb1eq3} gives
\begin{align}
    \log\left(\frac{1}{4\delta}\right) \leq \sum_{(i,j)\in\mathcal{K}} E_{\nu}[N_{i,j}(\tau_{\delta})]\, d(p_{i,j}, p'_{i,j}). \label{lb2eq1}
\end{align}
We now choose an instance $\nu' \in \epsilon_{\text{alt}}(\nu)$ that gives a different answer than $\nu$ to achieve the infimum on the right-hand side of \eqref{lb2eq1}, yielding
\begin{align}
    \log\left(\frac{1}{4\delta}\right)
    &\leq \inf_{\nu' \in \epsilon_{\text{alt}}(\nu)} \sum_{(i,j)\in\mathcal{K}} E_{\nu}[N_{i,j}(\tau_{\delta})]\, d(p_{i,j}, p'_{i,j}) \label{lb3eq1} \\
    &= E_{\nu}[J(\tau_\delta)] \inf_{\nu' \in \epsilon_{\text{alt}}(\nu)} \sum_{(i,j)\in\mathcal{K}} \frac{E_{\nu}[N_{i,j}(\tau_{\delta})]}{E_{\nu}[J(\tau_\delta)]}\, d(p_{i,j}, p'_{i,j}) \label{lb3eq2} \\
    &= E_{\nu}[J(\tau_\delta)] \inf_{\nu' \in \epsilon_{\text{alt}}(\nu)} \sum_{(i,j)\in\mathcal{K}} \frac{E_{\nu}[N_{i,j}(\tau_{\delta})]}{\sum_{(m,n)\in\mathcal{K}} (c_m+c_n)\, E_{\nu}[N_{m,n}(\tau_\delta)]}\, d(p_{i,j}, p'_{i,j}) \label{lb3eq3} \\
    &= E_{\nu}[J(\tau_\delta)] \inf_{\nu' \in \epsilon_{\text{alt}}(\nu)} \sum_{(i,j)\in\mathcal{K}} \frac{1}{c_i+c_j} \cdot \frac{(c_i+c_j)\, E_{\nu}[N_{i,j}(\tau_{\delta})]}{\sum_{(m,n)\in\mathcal{K}} (c_m+c_n)\, E_{\nu}[N_{m,n}(\tau_\delta)]}\, d(p_{i,j}, p'_{i,j}). \label{lb3eq4}
\end{align}
Define
\[
    w_{i,j} := \frac{(c_i+c_j)\, E_{\nu}[N_{i,j}(\tau_{\delta})]}{\sum_{(m,n)\in\mathcal{K}} (c_m+c_n)\, E_{\nu}[N_{m,n}(\tau_\delta)]},
\]
the proportion of cost incurred by pulling the arm pair $i$ and $j$. Note that $\sum_{(i,j)\in\mathcal{K}} w_{i,j} = 1$, giving
\begin{align}
    \log\left(\frac{1}{4\delta}\right) &\leq E_{\nu}[J(\tau_\delta)] \inf_{\nu' \in \epsilon_{\text{alt}}(\nu)} \sum_{(i,j)\in\mathcal{K}} \frac{w_{i,j}}{c_i+c_j}\, d(p_{i,j}, p'_{i,j}). \label{lb4eq1}
\end{align}
Taking the supremum over $\mathbf{w}$ further upper bounds the right-hand side:
\begin{align}
    \log\left(\frac{1}{4\delta}\right)
    &\leq E_{\nu}[J(\tau_\delta)] \sup_{\mathbf{w} \in \Sigma_{\mathcal{K}}} \inf_{\nu' \in \epsilon_{\text{alt}}(\nu)} \sum_{(i,j)\in\mathcal{K}} \frac{w_{i,j}}{c_i+c_j}\, d(p_{i,j}, p'_{i,j}) \nonumber \\
    &\leq E_{\nu}[J(\tau_\delta)]\, c^*(\nu)^{-1}. \label{lb5eq1}
\end{align}
Therefore, the lower bound on the total cost is
\begin{align}
    E_{\nu}[J(\tau_\delta)] \geq c^*(\nu) \log\left(\frac{1}{4\delta}\right), \label{lb6eq1}
\end{align}
which completes the proof.
\end{proof}

\begin{proof}[Proof of Theorem \ref{theorem: minmax}]
Without loss of generality, assume the Condorcet winner is arm $1$. The bound in Theorem \ref{theorem: minmax} is attained by optimally allocating the weights $w_{i,j}$ across the set of arms. To obtain an interpretable lower bound, we compute the optimal cost proportions by solving the optimization problem motivated by Lemma \ref{lemma:lower_bound_abstract}. Simplifying Equation \ref{eq:lower_bound_generic} gives the following discrete minimization form:
\begin{align}
    c^*(\nu)^{-1}
    &= \sup_{\mathbf{w} \in \Sigma_{\mathcal{K}}} \inf_{\nu' = (\mathcal{C},\mathbb{P}') \in \epsilon_{\text{alt}}(\nu)} \sum_{(i,j) \in \mathcal{K}} \frac{w_{i,j}}{c_i + c_j}\, d(p_{i,j}, p'_{i,j}) \label{eq1} \\
    &= \sup_{\mathbf{w} \in \Sigma_{\mathcal{K}}} \min_{i \neq 1} \sum_{j \neq i} \frac{w_{i,j}}{c_i + c_j}\, d(p_{i,j}, \max(0.5, p_{i,j})), \label{eq2}
\end{align}
where the second line reformulates the inner infimum as a discrete minimization over suboptimal arms.

The optimal weight allocation is given by
\begin{align}
    w(\nu) &= \argmax_{\mathbf{w} \in \Sigma_{\mathcal{K}}} \min_{i \neq 1} \sum_{j \in \mathcal{L} \setminus i} \frac{w_{i,j}}{c_i + c_j}\, d(p_{i,j}, \max(0.5, p_{i,j})). \label{mssr2eq1}
\end{align}
At the optimum, the quantities inside the minima, namely $\sum_{j \in \mathcal{L} \setminus i} \frac{w_{i,j}}{c_i + c_j}\, d(p_{i,j}, \max(0.5, p_{i,j}))$, are all equal. Thus, $w$ can be calculated by simply solving $L$ linear equations: $L-1$ from the inner minima of \eqref{mssr2eq1}, and one from the constraint on the weight vector. However, the optimal weights may not be unique for instances where there exists $i \neq 1$ such that there exist $j_1, j_2$ satisfying
\begin{align}
    \frac{d(p_{i,j_1}, 0.5)}{c_i + c_{j_1}} = \frac{d(p_{i,j_2}, 0.5)}{c_i + c_{j_2}} = \max_{j \neq i} \frac{d(p_{i,j}, 0.5)}{c_i + c_j}. \label{mssr3eq1}
\end{align}
Define $\mathcal{W}^*(\nu)$ to be the set of weight vectors that attain the supremum in \eqref{mssr2eq1}, and let $w_i := \sum_{j \in \mathcal{L} \setminus i} w_{i,j}$ denote the sum of weights involving arm $i$. We then have
\begin{align}
    w_i &= \frac{1}{\displaystyle\sum_{m \in \mathcal{L}} \frac{\max_{j \neq i}\left( \frac{1}{c_i + c_j}\, d(p_{i,j}, \max(0.5, p_{i,j})) \right)}{\max_{n \neq m}\left( \frac{1}{c_m + c_n}\, d(p_{m,n}, \max(0.5, p_{m,n})) \right)}}.
    \label{mssr4eq1}
\end{align}
Next, for every $i$ not satisfying \eqref{mssr3eq1}, it is easy to see that
\begin{align*}
    w_{i,j} =
    \begin{cases}
        0, & \text{if } \displaystyle\max_{k \neq i}\left( \frac{1}{c_i+c_k}\, d(p_{i,k}, \max\{0.5, p_{i,k}\}) \right) > \frac{1}{c_i+c_j}\, d(p_{i,j}, \max\{0.5, p_{i,j}\}), \\[6pt]
        w_i, & \text{otherwise.}
    \end{cases}
\end{align*}
For all values of $i$ satisfying \eqref{mssr3eq1}, the weight $w_i$ can be distributed as any linear combination of weights corresponding to all $j$ satisfying \eqref{mssr3eq1}. Finally, substituting one of the weight vectors in $\mathcal{W}^*(\nu)$ into Equation \ref{eq2} yields the desired result.
\end{proof}


\subsection{Mapping cost fraction to pull fraction}
\begin{proof}[Proof of Lemma \ref{lemma: mssr5eq1}:]
Let $\mathcal{W}_{i,j}$ be the optimal cost fraction and $\alpha_{i,j}$ be the corresponding pull fraction for the arm pair $i, j$. Then, for sufficiently large $t$, we have $N_{i,j}\left(t\right) = t\alpha_{i,j}$ and $J_{i,j}\left(t\right) = J\left(t\right) \mathcal{W}_{i,j}$, where $J_{i,j}\left(t\right)$ denotes the cost incurred by the arm pair $i, j$ and $J\left(t\right)$ represents the total cost incurred by the algorithm, As a result, we obtain
\begin{align}
    \mathcal{W}_{i,j}\left(t\right) &= \dfrac{J_{i,j}\left(t\right)}{J\left(t\right)} 
    \\&= \dfrac{N_{i,j}\left(t\right)\left(c_i + c_j\right)}{\sum_{(m,n)\in\mathcal{K}} N_{m,n}\left(t\right) \left(c_m+c_n\right)} 
    \\&= \dfrac{\alpha_{i,j}\left(t\right) \left(c_i+c_j\right)}{\sum_{(m,n)\in\mathcal{K}} \alpha_{m,n}\left(t\right) \left(c_m+c_n\right)} \label{sr51eq1}
\end{align}
Now, let $\theta\left(t\right) = \sum_{(m,n)\in\mathcal{K}} \alpha_{m,n}\left(t\right) \left(c_m+c_n\right)$
Therefore from \eqref{sr51eq1}, we have
\begin{align}
    \alpha_{i,j}\left(t\right) &= \dfrac{\mathcal{W}_{i,j}\left(t\right)}{c_i+c_j}\theta\left(t\right) \label{sr52eq1}
\end{align}
We also know that $\sum_{(m,n)\in\mathcal{K}} \alpha_{m,n}\left(t\right) = 1$ at all times, hence from \eqref{sr52eq1}, we know that
\begin{align}
    \sum_{(m,n)\in\mathcal{K}} \dfrac{\mathcal{W}_{i,j}\left(t\right)}{c_i+c_j} \theta\left(t\right) = 1 \label{sr53eq1}
\end{align}
Finally on using \eqref{sr52eq1} and \eqref{sr53eq1}, we have 
\begin{align}
    \alpha_{i,j}\left(t\right) = \dfrac{\mathcal{W}_{i,j}\left(t\right)/(c_i+c_j)}{\sum_{(m,n)\in\mathcal{K}} \mathcal{W}_{m,n}\left(t\right)/\left(c_m+c_n\right)}
\end{align}
\end{proof}

\subsection{Closed-form Solution for the GLR Statistic}

\begin{proof}[Proof of Lemma \ref{lemma: mssr1eq1}]
We start by identifying the unrestricted likelihood estimate of the log-likelihood function:
\begin{align}
    L(p) &= N^1_{i,j}(t) \log(p_{i,j}) + (N - N^1_{i,j}(t)) \log(1 - p_{i,j}) \label{sr11eq1} \\
    &= N_{i,j}(t)\left[ \hat{p}_{i,j}(t) \log(p_{i,j}) + (1 - \hat{p}_{i,j}(t)) \log(1 - p_{i,j}) \right]. \label{sr11eq2}
\end{align}
Differentiating \eqref{sr11eq2} and setting it to zero gives $p_{i,j} = \hat{p}_{i,j}(t)$. However, since $p_{i,j}$ is constrained with $0.5$ as the boundary value, we have
\begin{equation}
\label{sr12eq1}
    p_{i,j} =
    \begin{cases}
        \min(\hat{p}_{i,j}(t), 0.5), & \text{if } p_{i,j} < 0.5, \\
        \max(\hat{p}_{i,j}(t), 0.5), & \text{otherwise}.
    \end{cases}
\end{equation}
For brevity, define $q_{i,k}(t) := \max(\hat{p}_{i,k}(t), 0.5)$. Then
\begin{align}
    Z_{i,j}(t)
    &= \log \frac{\displaystyle\max_{i \text{ is best}} \prod_{m=1}^{L-1} \prod_{n=m+1}^L p_{m,n}^{N_{m,n}(t) \hat{p}_{m,n}(t)} (1 - p_{m,n})^{N_{m,n}(t)(1 - \hat{p}_{m,n}(t))}}{\displaystyle\max_{j \text{ is best}} \prod_{m=1}^{L-1} \prod_{n=m+1}^L p_{m,n}^{N_{m,n}(t) \hat{p}_{m,n}(t)} (1 - p_{m,n})^{N_{m,n}(t)(1 - \hat{p}_{m,n}(t))}}. \label{sr14eq1}
\end{align}
Substituting the constrained maximum-likelihood estimates from \eqref{sr12eq1}, the numerator and denominator simplify so that all factors involving arms other than $i$ or $j$ cancel. After simplification,
\begin{align}
    Z_{i,j}(t)
    &= \sum_{k \in \mathcal{L}} N_{j,k}(t) \log \left( \frac{\hat{p}_{j,k}(t)^{\hat{p}_{j,k}(t)} (1 - \hat{p}_{j,k}(t))^{1 - \hat{p}_{j,k}(t)}}{q_{j,k}(t)^{\hat{p}_{j,k}(t)} (1 - q_{j,k}(t))^{1 - \hat{p}_{j,k}(t)}} \right) \nonumber \\
    &\quad - \sum_{k \in \mathcal{L}} N_{i,k}(t) \log \left( \frac{\hat{p}_{i,k}(t)^{\hat{p}_{i,k}(t)} (1 - \hat{p}_{i,k}(t))^{1 - \hat{p}_{i,k}(t)}}{q_{i,k}(t)^{\hat{p}_{i,k}(t)} (1 - q_{i,k}(t))^{1 - \hat{p}_{i,k}(t)}} \right) \label{sr14eq3} \\
    &= \sum_{k \in \mathcal{L}} N_{j,k}(t)\, d(\hat{p}_{j,k}(t), q_{j,k}(t)) - \sum_{k \in \mathcal{L}} N_{i,k}(t)\, d(\hat{p}_{i,k}(t), q_{i,k}(t)). \label{sr14eq4}
\end{align}
Recalling $q_{i,k}(t) = \max(\hat{p}_{i,k}(t), 0.5)$ gives the stated closed form.
\end{proof}

\subsection{Correctness Argument}

\begin{definition}
\label{def: rank}
Consider a sequential identification problem specified by a partition $\mathcal{O} = \bigcup_{i=1}^{L} \mathcal{O}_i$, where each $\mathcal{O}_i$ represents an open set where arm $i$ is the correct answer. The \emph{rank} $R$ of the problem is the minimum number of arms required to describe the exclusion of any set $\mathcal{O}_i$ from $\mathcal{O}$. More specifically, the set $\mathcal{O} \setminus \mathcal{O}_i$ can be expressed as a finite union of conditions involving at most $R$ arms.
\end{definition}

\begin{lemma}
\label{lemma: delta}
For any sequential identification problem, let $R$ be its rank and $M$ the number of open sets. Then, by Proposition 21 of \cite{kaufmann2021mixture}, the Generalized Likelihood Ratio Test (GLRT) stopping rule is $\delta$-correct with threshold
\begin{align}
    \gamma(t, \delta) = 3R \log\left(1 + \log\left(\frac{t}{R}\right)\right) + R\, C\!\left(\frac{\log\left(\frac{M-1}{\delta}\right)}{R}\right). \label{delta}
\end{align}
\end{lemma}

\begin{proof}[Proof of Theorem \ref{theorem: delta-pac}]
The probability that the algorithm stops and outputs the wrong answer is bounded by
\begin{align}
    P_{\nu}\left(\tau_\delta < \infty, \hat{a}_{\tau_\delta} \neq a^*(\nu)\right)
    &\leq P_{\nu}\!\left(\exists b \in \mathcal{L} \setminus a^*(\nu), \exists t \in \mathbb{N} \,:\, \forall i \in \mathcal{L} \setminus b,\; Z_{b,i}(t) > \beta(t, \delta)\right) \label{pc1eq1} \\
    &\leq P_{\nu}\!\left(\exists b \in \mathcal{L} \setminus a^*(\nu), \exists t \in \mathbb{N} \,:\, Z_{b, a^*(\nu)}(t) > \beta(t, \delta)\right). \label{pc1eq2}
\end{align}
For the event in \eqref{pc1eq2}, we have
\begin{align}
    Z_{b, a^*(\nu)}(t)
    &= \sum_{k \in \mathcal{L}} N_{a^*(\nu),k}(t)\, d\!\left(\hat{p}_{a^*(\nu),k}(t), \max(\hat{p}_{a^*(\nu),k}(t), 0.5)\right) \nonumber \\
    &\quad - \sum_{k \in \mathcal{L}} N_{b,k}(t)\, d\!\left(\hat{p}_{b,k}(t), \max(\hat{p}_{b,k}(t), 0.5)\right) \label{pc2eq1} \\
    &\leq \sum_{k \in \mathcal{L}} N_{a^*(\nu),k}(t)\, d\!\left(\hat{p}_{a^*(\nu),k}(t), \max(\hat{p}_{a^*(\nu),k}(t), 0.5)\right). \label{pc2eq2}
\end{align}
Furthermore,
\begin{equation}
\label{pc3eq1}
    N_{a^*(\nu),k}(t)\, d\!\left(\hat{p}_{a^*(\nu),k}(t), \max(\hat{p}_{a^*(\nu),k}(t), 0.5)\right) =
    \begin{cases}
        N_{a^*(\nu),k}(t)\, d(\hat{p}_{a^*(\nu),k}(t), 0.5), & \text{if } \hat{p}_{a^*(\nu),k}(t) < 0.5, \\
        0, & \text{otherwise}.
    \end{cases}
\end{equation}
Therefore,
\begin{align}
    Z_{b, a^*(\nu)}(t) &\leq \sum_{k \in \mathcal{L}} N_{a^*(\nu),k}(t)\, d\!\left(\hat{p}_{a^*(\nu),k}(t), p_{a^*(\nu),k}\right). \label{pc4eq1}
\end{align}
Hence,
\begin{align}
    P_{\nu}\!\left(\tau_\delta < \infty, \hat{a}_{\tau_\delta} \neq a^*(\nu)\right)
    &\leq P_{\nu}\!\left(\exists b \in \mathcal{L} \setminus a^*(\nu), \exists t \in \mathbb{N} \,:\, \sum_{k \in \mathcal{L}} N_{a^*(\nu),k}(t)\, d\!\left(\hat{p}_{a^*(\nu),k}(t), p_{a^*(\nu),k}\right) > \beta(t, \delta)\right). \label{pc5eq1}
\end{align}
By Definition \ref{def: rank}, our problem has rank $R = L-1$ and number of open sets $M = L$. Substituting these into \eqref{delta}, the expression for $\gamma(t, \delta)$ matches $\beta(t, \delta)$ exactly, giving
\begin{align}
    P_{\nu}\!\left(\tau_\delta < \infty, \hat{a}_{\tau_\delta} \neq a^*(\nu)\right) &\leq \delta. \label{pc6eq1}
\end{align}
\end{proof}

\subsection{Cost Complexity}

\begin{lemma}
\label{lemma: hemicontinous}
The mapping $\nu \mapsto \alpha^*(\nu)$ is upper hemicontinuous and compact-valued. Furthermore, the set $\alpha^*(\nu)$ is convex for every instance $\nu \in \epsilon$.
\end{lemma}

\begin{proof}[Proof of Lemma \ref{lemma: hemicontinous}]
To establish the upper hemicontinuity and compactness of the mapping $\nu \mapsto \alpha^*(\nu)$, we begin by proving the compactness of $\mathcal{W}^*(\nu)$. Recall that
\begin{align}
    \mathcal{W}^*(\nu) &= \argmax_{\mathbf{w} \in \Sigma_{\mathcal{K}}} \min_{i \neq 1} \sum_{j \in \mathcal{L} \setminus i} \frac{w_{i,j}}{c_i + c_j}\, d(p_{i,j}, \max(0.5, p_{i,j})). \label{hc1eq1}
\end{align}
Since $\Sigma_{\mathcal{K}}$ is a $K$-dimensional simplex (closed and bounded, hence compact), and the objective is continuous, any set of maximizers $\mathcal{W}^*(\nu) \subseteq \Sigma_{\mathcal{K}}$ is also compact. Given the bijection from $\mathcal{W}^*(\nu)$ to $\alpha^*(\nu)$ via the continuous map $f(\mathcal{W}^*, \mathcal{C})$ (with continuous inverse), $\alpha^*(\nu)$ is also compact.

Next, define
\begin{align}
    \phi(\nu, \alpha)
    &= \min_{i \neq 1} \sum_{j \in \mathcal{L} \setminus i} \frac{\alpha_{i,j}}{\sum_{(m,n) \in \mathcal{K}} \alpha_{m,n} (c_m + c_n)}\, d(p_{i,j}, \max(0.5, p_{i,j})) \label{hc2eq1} \\
    &= \frac{1}{\sum_{(m,n) \in \mathcal{K}} \alpha_{m,n} (c_m + c_n)} \min_{i \neq 1} \sum_{j \in \mathcal{L} \setminus i} \alpha_{i,j}\, d(p_{i,j}, \max(0.5, p_{i,j})). \label{hc2eq2}
\end{align}
From \eqref{hc2eq2}, $\phi(\nu, \alpha)$ is continuous as a composition of continuous functions. By the definition of $\mathcal{W}^*(\nu)$ and the mapping $f(\mathcal{W}^*, \mathcal{C})$, we have
\begin{align}
    \alpha^*(\nu) &= \argmax_{\alpha \in \Sigma_{\mathcal{K}}} \min_{i \neq 1} \sum_{j \in \mathcal{L} \setminus i} \frac{\alpha_{i,j}}{\sum_{(m,n) \in \mathcal{K}} \alpha_{m,n}(c_m + c_n)}\, d(p_{i,j}, \max(0.5, p_{i,j})). \label{hc3eq1}
\end{align}
Since $\Sigma_{\mathcal{K}}$ is compact-valued (for every $\nu \in \epsilon$ there exists a feasible weight allocation) and $\phi(\nu, \alpha)$ is continuous in both arguments, Berge's maximum theorem implies that $\alpha^*(\nu)$ is upper hemicontinuous with nonempty, compact values.

We now show that $\alpha^*(\nu)$ is convex for all $\nu$. Fix $\nu$, pick $\alpha^{(1)}, \alpha^{(2)} \in \alpha^*(\nu)$, and let $z \in [0,1]$. By definition,
\begin{align}
    \phi(\nu, \alpha^{(1)}) = \phi(\nu, \alpha^{(2)}) = c^*(\nu)^{-1}. \label{hc4eq1}
\end{align}
Then
\begin{align}
    &\phi\!\left(\nu, z\alpha^{(1)} + (1-z)\alpha^{(2)}\right) \nonumber \\
    &= \min_{i \neq 1} \sum_{j \in \mathcal{L} \setminus i} \frac{z\alpha^{(1)}_{i,j} + (1-z)\alpha^{(2)}_{i,j}}{\sum_{(m,n) \in \mathcal{K}} \left(z\alpha^{(1)}_{m,n} + (1-z)\alpha^{(2)}_{m,n}\right) (c_m + c_n)}\, d(p_{i,j}, \max(0.5, p_{i,j})) \label{hc5eq1} \\
    &\geq z \min_{i \neq 1} \sum_{j \in \mathcal{L} \setminus i} \frac{\alpha^{(1)}_{i,j}}{\sum_{(m,n) \in \mathcal{K}} \left(z\alpha^{(1)}_{m,n} + (1-z)\alpha^{(2)}_{m,n}\right) (c_m + c_n)}\, d(p_{i,j}, \max(0.5, p_{i,j})) \nonumber \\
    &\quad + (1-z) \min_{i \neq 1} \sum_{j \in \mathcal{L} \setminus i} \frac{\alpha^{(2)}_{i,j}}{\sum_{(m,n) \in \mathcal{K}} \left(z\alpha^{(1)}_{m,n} + (1-z)\alpha^{(2)}_{m,n}\right) (c_m + c_n)}\, d(p_{i,j}, \max(0.5, p_{i,j})) \label{hc5eq2} \\
    &= z \min_{i \neq 1} \sum_{j \in \mathcal{L} \setminus i} \frac{w^{(1)}_{i,j}}{c_i + c_j}\, d(p_{i,j}, \max(0.5, p_{i,j})) + (1-z) \min_{i \neq 1} \sum_{j \in \mathcal{L} \setminus i} \frac{w^{(2)}_{i,j}}{c_i + c_j}\, d(p_{i,j}, \max(0.5, p_{i,j})) \label{hc5eq4} \\
    &= z\, c^*(\nu)^{-1} + (1-z)\, c^*(\nu)^{-1} \label{hc5eq5} \\
    &= c^*(\nu)^{-1}, \label{hc5eq6}
\end{align}
where the step from \eqref{hc5eq2} to \eqref{hc5eq4} uses the identity $\alpha_{m,n}(c_m+c_n) = w_{m,n}$ and a simplification of the normalising denominator. Since $\Sigma_{\mathcal{K}}$ is convex, $z\alpha^{(1)} + (1-z)\alpha^{(2)} \in \Sigma_{\mathcal{K}}$, so $\phi(\nu, z\alpha^{(1)} + (1-z)\alpha^{(2)}) \leq c^*(\nu)^{-1}$. Combining the two inequalities, $\phi(\nu, z\alpha^{(1)} + (1-z)\alpha^{(2)}) = c^*(\nu)^{-1}$. Hence $z\alpha^{(1)} + (1-z)\alpha^{(2)} \in \alpha^*(\nu)$, and $\alpha^*(\nu)$ is convex.
\end{proof}

For real-valued vectors $u, v$ of identical sizes and a compact, convex set $S$, we define
\[
    d_\infty(u, v) := \max_i |u_i - v_i|, \qquad d_\infty(u, S) := \min_{v \in S} d_\infty(u, v), \qquad \|u - v\| := \sqrt{\sum_i (u_i - v_i)^2}.
\]

\begin{lemma}
\label{lemma: tracking}
Under the tracking rule given in lines 3 to 7 of Algorithm \ref{algo: DCTAS}, almost surely
\begin{align}
    \lim_{t \to \infty} d_\infty\!\left((N_{i,j}(t)/t)_{i,j \in \mathcal{L}}, \alpha^*(\nu)\right) = 0. \label{l2eq1}
\end{align}
\end{lemma}

\begin{proof}[Proof of Lemma \ref{lemma: tracking}]
Let $S := \alpha^*(\nu)$ and fix $\epsilon > 0$. From the sampling rule, $N_{i,j}(t) = \Omega(\sqrt{t})$ almost surely for $i, j \in \mathcal{L}$. By the strong law of large numbers, $\|\hat{\nu}(t) - \nu\| \to 0$ almost surely. Since $\alpha(t+1) \in \alpha^*(\hat{\nu})$ for all $t$ and $\alpha^*(\hat{\nu})$ is close to $S$ for large $t$, by upper hemicontinuity (Lemma \ref{lemma: hemicontinous}) there exists $t_0 = t_0(\epsilon)$ such that $d_\infty(\alpha(t), S) \leq \epsilon$ for all $t \geq t_0$. Letting $v(t) := \argmin_{\alpha \in S} d_\infty(\alpha(t), \alpha)$ denote the projection of $\alpha(t)$ onto $S$, we have $d_\infty(\alpha(t), v(t)) \leq \epsilon$ for $t \geq t_0$.

Define $\bar{v}(t) := \frac{1}{t} \sum_{s=1}^t v(s)$ and $\bar{\alpha}(t) := \frac{1}{t} \sum_{s=1}^t \alpha(s)$. Then $\|\bar{\alpha}(t) - \bar{v}(t)\| \leq 9\epsilon$ for all $t > t_0' = t_0/(8\epsilon)$ (see \cite{reddy2023best}). Letting $\hat{\alpha}(t) := \argmin_{\alpha \in S} d_\infty(\bar{\alpha}(t), \alpha)$, we obtain $d_\infty(\bar{\alpha}(t), \hat{\alpha}(t)) \leq 9\epsilon$ for all $t \geq t_0'$.

Define $\epsilon_{i,j}(t) := N_{i,j}(t) - t \bar{\alpha}_{i,j}(t)$ for $i, j \in \mathcal{L}$. Then
\begin{align}
    \epsilon_{i,j}(t+1)
    &= N_{i,j}(t+1) - (t+1) \bar{\alpha}_{i,j}(t+1) \label{le1eq1} \\
    &= N_{i,j}(t) + \mathbb{1}(a_{t+1} = (i,j)) - (t+1) \bar{\alpha}_{i,j}(t+1) \label{le1eq2} \\
    &= \epsilon_{i,j}(t) + \mathbb{1}(a_{t+1} = (i,j)) + t \bar{\alpha}_{i,j}(t) - (t+1) \bar{\alpha}_{i,j}(t+1) \label{le1eq3} \\
    &= \epsilon_{i,j}(t) + \mathbb{1}(a_{t+1} = (i,j)) - \alpha_{i,j}(t+1) \label{le1eq4} \\
    &\leq \epsilon_{i,j}(t) + \mathbb{1}(a_{t+1} = (i,j)). \label{le1eq5}
\end{align}
We next show that there exists $t_0'' = t_0''(\epsilon)$ such that for all $t \geq t_0''$, $\{a_{t+1} = (i,j)\} \subset \{\epsilon_{i,j}(t) \leq 9t\epsilon\}$. From the sampling rule, $\{a_{t+1} = (i,j)\} \subset \chi_1(t) \cup \chi_2(t)$, where
\begin{align}
    \chi_1(t) &= \left\{ (i,j) = \argmin_{m,n \in \mathcal{L}} N_{m,n}(t) - \sum_{s=1}^t \alpha^*_{m,n}(s) \right\}, \label{le2eq1} \\
    \chi_2(t) &= \left\{ (i,j) = \argmin_{m,n \in \mathcal{L}} N_{m,n}(t) < \sqrt{t} \right\}. \label{le2eq2}
\end{align}
If $\chi_1(t)$ holds, then for all $t \geq t_0'$,
\begin{align}
    \epsilon_{i,j}(t) &= N_{i,j}(t) - t \bar{\alpha}_{i,j}(t) \overset{\eqref{le2eq1}}{=} \min_{m,n \in \mathcal{L}} N_{m,n}(t) - \sum_{s=1}^t \alpha^*_{m,n}(s) \leq 0 \leq 9t\epsilon. \label{le3eq3}
\end{align}
If $\chi_2(t)$ holds with $t > 1/\epsilon^2$, then
\begin{align}
    \epsilon_{i,j}(t) &= N_{i,j}(t) - t \bar{\alpha}_{i,j}(t) \overset{\eqref{le2eq2}}{\leq} \sqrt{t} - t \bar{\alpha}_{i,j}(t) \leq \frac{t}{\sqrt{t}} \leq 9t\epsilon. \label{le12eq7}
\end{align}
Hence,
\begin{align}
    \epsilon_{i,j}(t+1) &\leq \epsilon_{i,j}(t) + \mathbb{1}(\epsilon_{i,j}(t) \leq 9t\epsilon). \label{le4eq1}
\end{align}
We prove by induction that $\epsilon_{i,j}(t) \leq \max\{\epsilon_{i,j}(t_0''), 9t\epsilon + 1\}$. The base case $t = t_0''$ is immediate. For the inductive step, if $\epsilon_{i,j}(t) \leq 9t\epsilon$,
\begin{align}
    \epsilon_{i,j}(t+1) &\leq \epsilon_{i,j}(t) + 1 \leq 9t\epsilon + 1 \leq \max\{\epsilon_{i,j}(t_0''), 9(t+1)\epsilon + 1\}. \label{le5eq4}
\end{align}
If instead $\epsilon_{i,j}(t) > 9t\epsilon$, then by \eqref{le4eq1},
\begin{align}
    \epsilon_{i,j}(t+1) &\leq \epsilon_{i,j}(t) \leq \max\{\epsilon_{i,j}(t_0''), 9t\epsilon + 1\} \leq \max\{\epsilon_{i,j}(t_0''), 9(t+1)\epsilon + 1\}. \label{le6eq3}
\end{align}
Since $\sum_{(m,n) \in \mathcal{L}} \epsilon_{m,n}(t) = 0$,
\begin{align}
    \epsilon_{i,j}(t) &= -\sum_{(i',j') \in (\mathcal{L} \times \mathcal{L}) \setminus (i,j)} \epsilon_{i',j'}(t) \label{le7eq1} \\
    &\geq -\binom{L}{2} - 1) \max_{(i',j') \neq (i,j)} \max\{\epsilon_{i',j'}(t_0''), 9t\epsilon + 1\}. \label{le7eq2}
\end{align}
Therefore,
\begin{align}
    |\epsilon_{i,j}(t)|
    &\leq \left(\binom{L}{2} - 1\right) \max_{(i',j') \neq (i,j)} \max\{\epsilon_{i',j'}(t_0''), 9t\epsilon + 1\} \label{le8eq1} \\
    &\leq \left(\binom{L}{2} - 1\right) \max\{t_0'', 9t\epsilon + 1\} \label{le8eq2} \\
    &= \left(\binom{L}{2} - 1\right) t \max\!\left\{ \frac{t_0''}{t}, 9\epsilon + \frac{1}{t} \right\} \label{le8eq3} \\
    &\leq 10 \left(\binom{L}{2} - 1\right) t \epsilon \label{le8eq4}
\end{align}
for all $t \geq t_1 = \frac{1}{\epsilon} \max\{t_0'', 10\}$. This gives
\begin{align}
    d_\infty\!\left((N_{i,j}(t)/t)_{i,j \in \mathcal{L}}, S\right)
    &\leq d_\infty\!\left((N_{i,j}(t)/t)_{i,j \in \mathcal{L}}, \hat{\alpha}(t)\right) \label{le11eq1} \\
    &\leq d_\infty\!\left((N_{i,j}(t)/t)_{i,j \in \mathcal{L}}, \bar{\alpha}(t)\right) + d_\infty(\bar{\alpha}(t), \hat{\alpha}(t)) \label{le11eq2} \\
    &\leq \max_{i,j \in \mathcal{L}} \left| \frac{N_{i,j}(t)}{t} - \bar{\alpha}_{i,j}(t) \right| + 9\epsilon \label{le11eq3} \\
    &= \max_{i,j \in \mathcal{L}} \frac{|\epsilon_{i,j}(t)|}{t} + 9\epsilon \label{le11eq4} \\
    &\leq 10\epsilon \left(\binom{L}{2} - 1\right) + 9\epsilon. \label{le11eq5}
\end{align}
Since $\epsilon > 0$ was arbitrary, the result follows.
\end{proof}

\begin{lemma}
\label{lemma: lambert}
For any constants $k_1 > 0$, $k_2 > 0$ and every $\gamma \geq 1$, by \cite{garivier2016optimal},
\begin{align}
    \inf\left\{ t \geq 1 : t k_1 \geq \log(k_2 t^\alpha) \right\} \leq \frac{\gamma}{k_1}\!\left( \log \frac{k_2 e}{k_1^\alpha} + \log \log \frac{k_2}{k_1^\alpha} \right). \label{la1eq1}
\end{align}
\end{lemma}

\begin{lemma}
\label{lemma: C}
Given $C(x)$ as defined in Definition \ref{def: C_exp}, as $\delta \to 0$,
\begin{align}
    C\!\left( \frac{\log\left( \frac{L-1}{\delta} \right)}{L-1} \right) \leq (1 + \upsilon)\!\left( \frac{\log\left( \frac{L-1}{\delta} \right)}{L-1} + 4 \log\!\left( 1 + \frac{\log\left( \frac{L-1}{\delta} \right)}{L-1} + \sqrt{\frac{2 \log\left( \frac{L-1}{\delta} \right)}{L-1}} \right) \right),
\end{align}
where $\upsilon \to 0$ as $x \to \infty$.
\end{lemma}

\begin{proof}[Proof of Lemma \ref{lemma: C}]
Recall that
\begin{align}
    C(x) &= 2\tilde{h}_{3/2}\!\left( \frac{h^{-1}(1+x) + \log(\pi^2/3)}{2} \right) \label{sc5eq1} \\
    &\leq 2\tilde{h}_{3/2}\!\left( \frac{1 + x + \log(1 + x + \sqrt{2x}) + \log(\pi^2/3)}{2} \right) \label{sc5eq2} \\
    &\leq 2\tilde{h}_{3/2}\!\left( \frac{x + 2 \log(1 + x + \sqrt{2x})}{2} \right), \label{sc5eq3}
\end{align}
where the last inequality holds for all $x$ such that $\log(1 + x + \sqrt{2x}) > \log(\pi^2/3) + 1$. Now,
\begin{align}
    h^{-1}\!\left( \frac{x + 2 \log(1 + x + \sqrt{2x})}{2} \right)
    &\leq \frac{x + 2 \log(1 + x + \sqrt{2x})}{2} \nonumber \\
    &\quad + \log\!\left( \frac{x + 2 \log(1 + x + \sqrt{2x})}{2} + \sqrt{x + 2 \log(1 + x + \sqrt{2x}) - 2} \right) \label{sc6eq2} \\
    &\leq \frac{x + 2 \log(1 + x + \sqrt{2x})}{2} + \log(1 + x + \sqrt{2x}) \label{sc6eq3} \\
    &= \frac{x + 4 \log(1 + x + \sqrt{2x})}{2}, \label{sc6eq4}
\end{align}
where \eqref{sc6eq3} holds for all $x > k_1$ for some $k_1 \in \mathbb{N}^+$. From the definition of $\tilde{h}_z(y)$, for every $y > h(\log(1/z))$ we have $\tilde{h}_z(y) = e^{1/h^{-1}(y)} h^{-1}(y)$. Hence, there exists $k_2 \in \mathbb{N}^+$ such that for all $x > k_2$,
\begin{align}
    C(x)
    &\leq 2\tilde{h}_{3/2}\!\left( \frac{x + 2 \log(1 + x + \sqrt{2x})}{2} \right) \label{sc12eq1} \\
    &\leq 2 e^{\frac{2}{x + 4 \log(1 + x + \sqrt{2x})}} \cdot \frac{x + 4 \log(1 + x + \sqrt{2x})}{2} \label{sc12eq2} \\
    &\leq (1 + \upsilon)\!\left( x + 4 \log(1 + x + \sqrt{2x}) \right), \label{sc12eq3}
\end{align}
where $\upsilon \to 0$ as $x \to \infty$ and \eqref{sc12eq2} holds for $x > k_3$ for some $k_3 \in \mathbb{N}^+$. Letting $x = \log\!\left( \frac{L-1}{\delta} \right) / (L-1)$ and $k = \max\{k_1, k_2, k_3\}$, as $\delta \to 0$, $x$ exceeds $k$, and substituting into \eqref{sc12eq3} yields the desired result.
\end{proof}

\begin{proof}[Proof of Theorem \ref{theorem: sample complexity}]
Consider the event
\begin{align}
    \chi = \left\{ d_\infty\!\left( (N_{i,j}(t)/t)_{i,j \in \mathcal{L}}, \alpha^*(\nu) \right) \xrightarrow{t \to \infty} 0,\; \hat{\nu}(t) \xrightarrow{t \to \infty} \nu \right\}. \label{sc1eq1}
\end{align}
By Lemma \ref{lemma: tracking}, $\chi$ occurs with probability $1$ under the DCTAS algorithm. By the continuity of $\phi$, there exists an open neighborhood $U = U(\epsilon)$ of $\{\nu\} \times \alpha^*(\nu)$ such that
\begin{align}
    \phi(\nu_0, \alpha_0)
    &\geq (1+\epsilon)^{-1} \phi(\nu, \alpha^*) \label{sc2eq1} \\
    &= (1+\epsilon)^{-1} \min_{i \neq 1} \sum_{j \in \mathcal{L} \setminus i} \frac{\alpha_{i,j}}{\sum_{(m,n) \in \mathcal{K}} \alpha_{m,n}(c_m + c_n)}\, d(p_{i,j}, \max(0.5, p_{i,j})) \label{sc2eq2} \\
    &= (1+\epsilon)^{-1} \min_{i \neq 1} \sum_{j \in \mathcal{L} \setminus i} \frac{w_{i,j}}{c_i + c_j}\, d(p_{i,j}, \max(0.5, p_{i,j})) \label{sc2eq3} \\
    &= (1+\epsilon)^{-1} c^*(\nu)^{-1} \quad \forall (\nu_0, w_0) \in U. \label{sc2eq4}
\end{align}
Under $\chi$, there exists $n_0 = n_0(\epsilon)$ such that $\left( \hat{\nu}(t), (N_{i,j}(t)/t)_{i,j \in \mathcal{L}} \right) \in U$ for all $t \geq n_0$, so
\begin{align}
    \phi\!\left( \hat{\nu}(t), (N_{i,j}(t)/t)_{i,j \in \mathcal{L}} \right) &\geq (1+\epsilon)^{-1} c^*(\nu)^{-1} \quad \forall t \geq n_0. \label{sc3eq1}
\end{align}
Since $\hat{\nu}(t) \to \nu$ and the best arm is unique, there exists $n_1$ such that for all $t \geq n_1$, the GLRT statistic can be written as
\begin{align}
    Z(t)
    &= \min_{j \neq 1} \left( \sum_{k \in \mathcal{L} \setminus j} N_{j,k}(t)\, d(\hat{p}_{j,k}(t), \max(\hat{p}_{j,k}(t), 0.5)) \right. \nonumber \\
    &\qquad\qquad \left. - \sum_{k \in \mathcal{L} \setminus 1} N_{1,k}(t)\, d(\hat{p}_{1,k}(t), \max(\hat{p}_{1,k}(t), 0.5)) \right) \label{sc4eq1} \\
    &= \min_{j \neq 1} \sum_{k \in \mathcal{L} \setminus j} N_{j,k}(t)\, d(\hat{p}_{j,k}(t), \max(\hat{p}_{j,k}(t), 0.5)) \label{sc4eq2} \\
    &= t \min_{j \neq 1} \sum_{k \in \mathcal{L} \setminus j} \frac{N_{j,k}(t)}{t}\, d(\hat{p}_{j,k}(t), \max(\hat{p}_{j,k}(t), 0.5)) \label{sc4eq3} \\
    &= J(t) \min_{j \neq 1} \sum_{k \in \mathcal{L} \setminus j} \frac{N_{j,k}(t)/t}{\sum_{(m,n) \in \mathcal{K}} N_{m,n}(t)(c_m + c_n)/t}\, d(\hat{p}_{j,k}(t), \max(\hat{p}_{j,k}(t), 0.5)) \label{sc4eq4} \\
    &= J(t)\, \phi\!\left( \hat{\nu}(t), (N_{j,k}(t)/t)_{j,k \in \mathcal{L}} \right) \label{sc4eq5} \\
    &\overset{\eqref{sc3eq1}}{\geq} J(t) (1+\epsilon)^{-1} c^*(\nu)^{-1}. \label{sc4eq6}
\end{align}
Let $\log(B) = 4(L-1) \log\!\left( 1 + \frac{\log\left( \frac{L-1}{\delta} \right)}{L-1} \right) + \sqrt{\frac{2 \log\left( \frac{L-1}{\delta} \right)}{L-1}}$. From the definition of $\beta(t, \delta)$ and Lemma \ref{lemma: C},
\begin{align}
    \beta(t, \delta)
    &\leq 3(L-1) \log\!\left( 1 + \log\!\left( \frac{t}{L-1} \right) \right) + (1+\upsilon)\!\left( \log\!\left( \frac{L-1}{\delta} \right) + \log(B) \right) \label{sc7eq1} \\
    &\leq (1+\upsilon)\!\left( 3(L-1) \log\!\left( 1 + \log\!\left( \frac{t}{L-1} \right) \right) + \log\!\left( \frac{L-1}{\delta} \right) + \log(B) \right) \label{sc7eq3} \\
    &\leq (1+\upsilon)\!\left( 3(L-1) \log\!\left( \frac{t}{L-1} \right)^{\frac{1}{3(L-1)}} + \log\!\left( \frac{L-1}{\delta} \right) + \log(B) \right) \label{sc7eq2} \\
    &= (1+\upsilon) \log\!\left( \frac{Bt}{\delta} \right), \label{s8eq2}
\end{align}
where \eqref{sc7eq2} holds for all $t > n_2$ for some $n_2 \in \mathbb{N}^+$. We then have
\begin{align}
    J(\tau_\delta)
    &= J\!\left( \inf\{ t \geq \max\{n_0, n_1, n_2\} : Z(t) \geq \beta(t, \delta) \} \right) \label{sc9eq1} \\
    &\overset{\eqref{s8eq2}}{\leq} J\!\left( \inf\!\left\{ t \geq \max\{n_0, n_1, n_2\} : J(t) (1+\epsilon)^{-1} (1+\upsilon)^{-1} c^*(\nu)^{-1} \geq \log\!\left( \frac{Bt}{\delta} \right) \right\} \right) \label{sc9eq2} \\
    &\leq J(n_0) \cup J(n_1) \cup J(n_2) \nonumber \\
    &\quad \cup\, J\!\left( \inf\!\left\{ t : J(t) (1+\epsilon)^{-1} (1+\upsilon)^{-1} c^*(\nu)^{-1} \geq \log\!\left( \frac{B J(t)}{\min(\mathcal{C}) \delta} \right) \right\} \right). \label{sc9eq3}
\end{align}
By Lemma \ref{lemma: lambert},
\begin{align}
    &J\!\left( \inf\!\left\{ t : J(t) (1+\epsilon)^{-1} (1+\upsilon)^{-1} c^*(\nu)^{-1} \geq \log\!\left( \frac{B J(t)}{\min(\mathcal{C}) \delta} \right) \right\} \right) \nonumber \\
    &\quad \leq (1+\epsilon)(1+\upsilon) c^*(\nu) \left[ \log\!\left( \frac{B e (1+\epsilon)(1+\upsilon) c^*(\nu)}{\min(\mathcal{C}) \delta} \right) + \log \log\!\left( \frac{B (1+\epsilon)(1+\upsilon) c^*(\nu)}{\min(\mathcal{C}) \delta} \right) \right] \nonumber \\
    &\qquad + \mathcal{O}(\max(\mathcal{C})). \label{sc10eq1}
\end{align}
From \eqref{sc9eq3} and \eqref{sc10eq1}, for every $\delta \in (0,1)$, letting $\epsilon$ and $\upsilon$ tend to $0$,
\begin{align}
    \limsup_{\delta \to 0} \frac{J(\tau_\delta)}{\log(1/\delta)} &\leq c^*(\nu). \label{sc11eq1}
\end{align}
\end{proof}
\section{Baseline Algorithms}
\label{app: baselines}

We now elaborate on the baseline algorithms used in the paper. Note that as the TAS algorithm is a simple adaptation of the DCTAS algorithm with $\mathcal{C} = (1, 1, \dots, 1)_L$. However, to calculate the total expected cost, we count the number of pulls corresponding to all arm pairs, and multiply it with the actual costs of the model, before summing up across all the pairs.

\subsection{Cost Aware Round Robin (CRR)}
The cost-aware round-robin algorithm is inspired by the Chernoff Overlap algorithm provided in \cite{kanarios2024cost} and features its adaptation to the dueling setting with a slightly modified sampling rule. In this algorithm, at each time step $t$, we select the arm pair that has incurred the least cost. Mathematically, if $N_{i,j}$ represents the number of times arm pair $i, j$ is pulled, then the pair we pull next has the minimum $\left(c_i + c_j\right)N_{i,j}$ as seen in Algorithm \ref{algo: crr}. Further, similar to DCTAS (refer to Algorithm \ref{algo: DCTAS}), we use the standard GLRT-based stopping rule to terminate the algorithm. 

\begin{algorithm}[ht]
\caption{Cost Aware Round Robin (CRR)}
\label{algo: crr}
\begin{algorithmic}
  \STATE \textbf{Input:} Error Threshold $\delta$; Cost Vector $\mathcal{C}$
  \STATE \textbf{Output:} Best arm $i \in \mathcal{L}$
\end{algorithmic}  
\begin{algorithmic}[1]
\STATE Pull each arm pairs $i, j \in \mathcal{L}$ once for initialization;
\FOR{$t=1, 2, \dots$}
\STATE $a_t = \argmin_{i,j \in \mathcal{L}} \left(c_i + c_j\right)N_{i,j}$;
\STATE Pull $a_t$ and update $\hat{p}_t$;
\IF{$\exists i \in \mathcal{L}, \forall j \in \mathcal{L} \setminus \{i\}, Z_{i,j}\left(t\right) > \beta\left(t,\delta\right)$}
\RETURN $i \in \mathcal{L}$
\ENDIF
\ENDFOR
\end{algorithmic}
\end{algorithm}

\begin{theorem}
    \label{theorem: delta-pac-crr}
    Let $\delta \in \left(0,1\right)$ and $\beta\left(t, \delta\right) = 3(L-1)\log\left(1 + \log\left(\frac{t}{L-1}\right)\right) + \left(L-1\right)C\left(\frac{\log\left(\frac{L-1}{\delta} \right)}{L-1}\right)$. Then, the CRR algorithm presented in Algorithm \ref{algo: crr} is $\delta$-PC i.e.,
    \begin{align*}
        P_{\nu}\left(\tau_\delta < \infty, \hat{a}_{\tau_\delta} \neq a^*\left(\nu\right)\right) \leq \delta
    \end{align*}
\end{theorem}
The proof of Theorem~\ref{theorem: delta-pac-crr} follows a similar structure to that of Theorem~\ref{theorem: delta-pac}, as both rely on the same GLRT-based stopping rule. The detailed proof is provided in Section~\ref{app: lower bound}.

\subsection{Dual Phase Cost Aware (DPCA)}
The Dual Phase Cost Aware algorithm is a natural extension of the two-phase cost unaware algorithm provided by \cite{karnin2016verification}. The algorithm as the name suggests, runs in two phases; the first is the exploration phase, followed by a verification phase.

\begin{algorithm}[ht]
\caption{Dual Phase Exploration}
\label{algo: exp}
\begin{algorithmic}
  \STATE \textbf{Input:} Error Threshold $\delta$; Cost Vector $\mathcal{C}$; Hyperparameter $k_1=10$
  \STATE \textbf{Output:} Candidate best arm $\hat{i} \in \mathcal{L}$, Adversary Pairs $i,j(i) \in \mathcal{L} \times \mathcal{L}$
\end{algorithmic}  
\begin{algorithmic}[1]
\STATE Initialize $\mathcal{Z} \leftarrow \mathcal{L} \times \mathcal{L}$; 
\STATE Pull each arm pairs $i, j \in \mathcal{L}$ once for initialization;
\FOR{$t = 1, 2, \dots$}
\IF{$|\mathcal{Z}| > 0$}
\FOR{all $\{i, j\} \mid \left(i, j\right) \in \mathcal{Z} \text{ or } \left(j, i\right) \in \mathcal{Z}$}
\STATE Pull the arm once;
\IF{Reward = 1}
\STATE $\mathcal{S}_{i,j} \leftarrow \mathcal{S}_{i,j} + 1$;
\ENDIF
\STATE Let $\hat{p}_{i,j} \leftarrow \frac{\mathcal{S}_{i,j}}{N_{i,j}} - \frac{1}{2}$;
\STATE Initialize $lcb_{i,j} \leftarrow \hat{p}_{i,j} - \sqrt{\frac{\ln{\left(2N_{i,j}^2L^2/\delta\right)}}{2N_{i,j}}}$ and $ucb_{i,j} \leftarrow \hat{p}_{i,j} + \sqrt{\frac{\ln{\left(2N_{i,j}^2L^2/\delta\right)}}{2N_{i,j}}}$;
\IF{$lcb_{i,j} > 0$ or $\left(ucb_{i,j} < 0 \text{ and } k_1ucb_{i,j} < lcb_{i,j}\right)$ or $\left(lcb_{i,j} > 0 \text{ and } \exists j' \in \mathcal{L} \backslash j \text{ s.t. } lcb_{i,j} > ucb_{i,j'} \right)$}
\STATE Remove arm pair $(i, j)$ from $\mathcal{Z}$;
\ENDIF
\ENDFOR
\ELSE
\RETURN $\hat{i} := i$ s.t. $lcb_{ij} \geq 0 \; \forall j \in \mathcal{L}$, $\forall i \in \mathcal{L} \backslash \hat{i}, j\left(i\right) = \argmin \frac{ucb_{i,j}}{\sqrt{c_i + c_j}} $;
\ENDIF
\ENDFOR
\end{algorithmic}
\end{algorithm}

In the exploration phase, as presented in Algorithm \ref{algo: exp}, we pull all the arm pairs in a round robin fashion until the stopping criteria (for that arm pair) is met. Note that the confidence parameter used in the exploration phase is independent of the error probability initialized by the user, and is often larger than the actual error probability. This explicit exploration phase ensures that for the actual $\delta$ (which is smaller than the exploration error probability), we have $L-1$ arm pairs to pull instead of $^LC_2$ as applicable in the standard round robin type algorithms. The algorithm then returns a potential Condorcet winner, along with arm pairs $\left(i, {j\left(i \right)}\right) \; \forall {i} \in \mathcal{L} \backslash \hat{i}$ such that the condition in Lemma \ref{lemma: dp} gets satisfied.

\begin{lemma}
\label{lemma: dp}
    Given the good event, when the Algorithm \ref{algo: exp} stops we have
    \begin{itemize}
        \item The Condorcet winner $i^*$ is the unique LLM for which for all $j \neq i^*$, we have $lcb_{i^*,j} > 0$
        \item For all LLMs $i \neq i^*$, let $y\left(i \right) = \argmin \frac{ucb_{i,y}}{\sqrt{c_i + c_y}}$. We  then have $\frac{p_{i, j\left(i\right)}}{\sqrt{c_i + c_{j\left(i\right)}}} \leq \frac{1}{2} \min_j \frac{p_{i,j}}{\sqrt{c_i + c_j}}$
    \end{itemize}
\end{lemma}

\textbf{Proof of Lemma \ref{lemma: dp}:}
For the Condorcet winner $i^*$ we have $p_{i^*,j} > 0 \; \forall j \neq i^*$, hence the only way a pair $\left(i^*,j\right)$ is eliminated from $\mathcal{Z}$ is when $lcb_{i^*,j}>0$. Also, for any non-winner $i$ there is some $j$ such that $p_{i,j} < 0$, hence it cannot be the case that $lcb_{i,j} > 0$. This proves the first point.

Now define $j'$ as $\argmin_j p_{i,j}$ for some $i \neq i^*$. We first show that the only way $j'$ gets eliminated is through the second elimination condition. Since $p_{i,j'} < 0$ it must be the case that $lcb_{i,j'} < 0$ at all times, hence it could not have been eliminated as per the first bullet. If $j'$ would have been eliminated according to the third bullet, we must have for some other $j$ that
\begin{align}
    p_{i,j'} > lcb_{i,j'} \geq ucb_{i,j} > p_{i,j}
\end{align}
which contradicts the definition of $j'$. Therefore, it follows that $j'$ is eliminated according to the second bullet, and at termination we have
\begin{align}
    ucb_{i,j'} < 0 ,\; ucb_{i,j'} < \frac{1}{2} lcb_{i,j'}
\end{align}
Now, according to the definition of $j\left(i\right)$, we have
\begin{align}
    \dfrac{p_{i, j\left(i\right)}}{\sqrt{c_i + c_{j\left(i\right)}}} \leq \dfrac{ucb_{i, j\left(i\right)}}{\sqrt{c_i + c_{j\left(i\right)}}} \leq \dfrac{ucb_{i, j'}}{\sqrt{c_i + c_{j'}}} \leq \dfrac{1}{2}\dfrac{lcb_{i, j'}}{\sqrt{c_i + c_{j'}}} \leq \dfrac{1}{2}\dfrac{p_{i, j'}}{\sqrt{c_i + c_{j'}}}
\end{align}

\begin{algorithm}[ht]
\caption{Dual Phase Verification}
\label{algo: ver}
\begin{algorithmic}
  \STATE \textbf{Input:} Error Threshold $\delta$; Candidate best arm $\hat{i} \in \mathcal{L}$, Adversary Pairs $i,j(i) \in \mathcal{L} \times \mathcal{L}$
  \STATE \textbf{Output:} Stop $\in $ \{True, False \}
\end{algorithmic}  
\begin{algorithmic}[1]
\STATE Initialize $\mathcal{A} \leftarrow \mathcal{L} \backslash \hat{i}$; 
\STATE Initialize Run = True ;
\IF{$|\mathcal{A}| > 0$ and Run == True}
\FOR{all the adversary pairs $\left(i, j\left(i\right)\right) \mid i \in \mathcal{A}$}
\STATE Pull the arm once;
\IF{Reward = 1}
\STATE $\mathcal{S}_{i,j\left(i\right)} \leftarrow \mathcal{S}_{i,j\left(i\right)} + 1$;
\ENDIF
\STATE Let $\hat{p}_{i,j\left(i\right)} \leftarrow \frac{\mathcal{S}_{i,j\left(i\right)}}{N_{i,j\left(i\right)}} - \frac{1}{2}$;
\STATE Initialize $lcb_{i,j\left(i\right)} \leftarrow \hat{p}_{i,j\left(i\right)} - \sqrt{\frac{\ln{\left(2N_{i,j\left(i\right)}^2L^2/\delta\right)}}{2N_{i,j\left(i\right)}}}$ and $ucb_{i,j\left(i\right)} \leftarrow \hat{p}_{i,j\left(i\right)} + \sqrt{\frac{\ln{\left(2N_{i,j\left(i\right)}^2L^2/\delta\right)}}{2N_{i,j\left(i\right)}}}$;
\IF{$ucb_{i,j\left(i\right)}$ < 0}
\STATE Remove $i$ from $\mathcal{A}$;
\ELSIF{$lcb_{i,j\left(i\right)}$ > 0}
\STATE Run = Fail;
\RETURN Stop = False
\ENDIF
\ENDFOR
\ELSIF{Run = True}
\RETURN Stop = True;
\ENDIF
\end{algorithmic}
\end{algorithm}

Next, the verification phase does a sanity check on the hypothesis made in the exploration phase, i.e., whether arm $j\left(i\right)$ indeed beats arm $i$ for all $i \in \mathcal{L} \backslash \hat{i} $. Incase, the hypothesis is true, the algorithm stops and returns the potential best arm $(\hat{i})$ as the Condorcet winner, otherwise it would rerun the exploration and the verification phase with a certain error probability as stated in Algorithm \ref{algo: DPCA}.

\begin{algorithm}[t]
\caption{Dual Phase Cost Aware Algorithm (DPCA)}
\label{algo: DPCA}
\begin{algorithmic}
  \STATE \textbf{Input:} Error Threshold $\delta$; Exploration Error $\delta'$; $\text{Dual Phase Exploration}\left(\delta, \mathcal{C}, k_1\right)$; $\text{Dual Phase Verification}\left(\delta, \hat{i}, \left(i, j\left(i\right)\right) \; \forall i \in \mathcal{L} \backslash \hat{i} \right)$ 
  \STATE \textbf{Output:} Best arm $i \in \mathcal{L}$
\end{algorithmic}  
\begin{algorithmic}[1]
\FOR{$r=1, 2, \cdots$}
\STATE $\hat{i}, \left(i, j\left(i\right) \right) \forall l_i \neq l_{\hat{i}} = \text{ Dual Phase Exploration}\left(\hat{\nu}\left(t\right),\mathcal{K}, \mathcal{C} \right)$; 
\STATE Stop = $\text{ Dual Phase Verification}\left(\hat{\nu}\left(t\right), \delta/r^2, l_{\hat{i}}, \left(i, j\left(i\right) \right) \forall l_i \neq l_{\hat{i}}\right)$;
\IF{Stop = True}
\RETURN $l_{\hat{i}}$;
\ENDIF
\ENDFOR
\end{algorithmic}
\end{algorithm}

\begin{theorem}
    \label{theorem: delta-pac-dpca}
    Let $\delta \in \left(0,1\right)$ then the DPCA algorithm presented in Algorithm \ref{algo: DPCA} is $\delta$-PC i.e.,
    \begin{align*}
        P_{\nu}\left(\tau_\delta < \infty, \hat{a}_{\tau_\delta} \neq a^*\left(\nu\right)\right) \leq \delta
    \end{align*}
\end{theorem}
The proof of Theorem \ref{theorem: delta-pac-dpca} is straightforward and follows Theorem 5 from \cite{karnin2016verification}.

\subsection{Dueling Bandits Cost Aware Track and Stop - Confidence Based (DCTAC)}
The DCTAC algorithm is an extension of the DCTAS algorithm with a confidence interval based stopping rule. While the sampling mechanism remains the decision to stop is governed by the statistical confidence in the estimated pairwise preferences. Once the algorithm determines that, with high probability, one arm consistently outperforms the others within the computed confidence bounds, it terminates and declares that arm as the winner.

The complete algorithmic steps are outlined in Algorithm~\ref{algo:DCTAC}.

\begin{algorithm}[ht]
\caption{Dueling Bandit Cost Aware Track and Stop - Confidence Based (DCTAC)}
\label{algo:DCTAC}
\begin{algorithmic}
  \STATE \textbf{Input:} Error Threshold $\delta$; 
  Cost Vector $\mathcal{C}$
  \STATE \textbf{Output:} Best arm $i \in \mathcal{L}$
\end{algorithmic}
\begin{algorithmic}[1]
\STATE Pull each arm pairs $i, j \in \mathcal{L}$ once for initialization;
\FOR{$t=1,2, \dots$}
\STATE $\alpha\left(t\right) = \text{WeightPullAllocation}\left(\mathcal{C}, \hat{\mathbb{P}}\left(t\right) \right)$;
\IF{$\exists k \in \mathcal{K} \mid N_{k}\left(t\right) < \sqrt{t}$}
\STATE $a_t = \argmin_{k \in \mathcal{K}} N_k\left(t\right)$;
\ELSE
\STATE $a_t = \argmin_{k \in \mathcal{K}} N_k\left(t\right) - \sum_{s=1}^t \alpha_{k}\left(s\right)$ ;
\ENDIF
\STATE Pull pair $a_t$
\IF{$\exists i \in \mathcal{L}$ such that $\hat{p}_{i,j} - \sqrt{\frac{\log\!\left(\frac{2 \pi^2 N_{i,j}^2 \binom{L}{2}}{6 \delta}\right)}{2 N_{i,j}}}
 > 0.5$ for all $j \in \mathcal{L} \setminus \{i\},$}
\RETURN $i$
\ENDIF
\ENDFOR
\end{algorithmic}
\end{algorithm}

\begin{theorem}
    \label{theorem: delta-pac-dctatc}
    Let $\delta \in \left(0,1\right)$ then the DCTAC algorithm presented in Algorithm \ref{algo:DCTAC} is $\delta$-PC i.e.,
    \begin{align*}
        P_{\nu}\left(\tau_\delta < \infty, \hat{a}_{\tau_\delta} \neq a^*\left(\nu\right)\right) \leq \delta
    \end{align*}
\end{theorem}
The proof follows directly from the application of Hoeffding’s inequality to bound the deviation of the empirical pairwise preferences from their true means. By applying a union bound over all arm pairs and time steps, it can be shown that the probability of an incorrect selection does not exceed~$\delta$.
\section{Real World Dataset}
\label{App: DD}

In this section, we provide a detailed overview of the real-world datasets used to evaluate the empirical performance of the DCTAS algorithm. As previously discussed, these datasets span four distinct tasks: text-to-text generation, text-to-image generation, vision-based tasks, and search-based tasks, and the dueling instances are formed using the preference data provided by Chatbot Arena \cite{chiang2024chatbot}. The same reference is also used to obtain the cost of models for the \textit{Text-to-Image}, \textit{Text-to-Text}, and \textit{Vision} datasets. On the other hand, for search-based tasks, the model costs are retrieved from their respective official webpage of the model. Since cost information is not available for all models listed on the leaderboard, we restrict our experiments to those for which cost values are publicly available. Additionally, for all the datasets excluding the text-to-image dataset, the reported cost is the cost required to generate one million tokens using the corresponding model. For the text-to-text and the vision models, we also assume that an average of $1000$ tokens are generated per query. A similar assumption is made for the Search dataset, where the cost corresponds to the High pricing tier based on input/output token and per-request rates. However, before proceeding with the description of individual datasets, we first outline the key assumptions underlying the dueling bandit framework and specify which of these assumptions are satisfied by the considered datasets.

To formally describe the assumptions \cite{bengs2021preference}, let $\Delta$ be the matrix such that, for all $i, j \in \mathcal{L}$, we define $\Delta_{i,j} = p_{i,j} - 0.5$. Using this definition, the assumptions are described as follows:

\begin{itemize}
    \item Low Noise Model (LNM) : For all $i, j \in \mathcal{L}$ if $\Delta_{i,j} > 0$, then it follows that $\sum_{k\in \mathcal{L}} \Delta_{i,k} > \sum_{k\in \mathcal{L}} \Delta_{j,k}$.
    \item Total Ordering (TO): For any $i \succ j$, we have $\Delta_{i,j} > 0$.  
    \item Strong Stochastic Transitivity (SST): For all triplets $i, j, k \in \mathcal{L}$, if $\Delta_{i,j}, \Delta_{j,k} > 0$, we have $\Delta_{i,k} \geq \max\{\Delta_{i,j}, \Delta_{j,k}\}$
    \item $\gamma-$ Relaxed Stochastic Transitivity ($\gamma-$RST): For all triplets $i, j, k \in \mathcal{L}$ and $\gamma \in \left(0,1\right)$, if $\Delta_{i,j}, \Delta_{j,k} > 0$, we have $\Delta_{i,k} \geq \gamma\max\{\Delta_{i,j}, \Delta_{j,k}\}$
    \item Moderate Stochastic Transitivity (MST): For all triplets $i, j, k \in \mathcal{L}$, if $\Delta_{i,j}, \Delta_{j,k} > 0$, we have $\Delta_{i,k} \geq \min\{\Delta_{i,j}, \Delta_{j,k}\}$
    \item Weak Stochastic Transitivity (WST): For all triplets $i, j, k \in \mathcal{L}$, $\Delta_{i,j}, \Delta_{j,k} > 0 \implies \Delta_{i,k} > 0$
    \item Stochastic Triangle Inequality (STI): Assuming total order over arms, for all $i \succ \ j \succ k$ we have $\Delta_{i,k} \leq \Delta_{i,j} + \Delta_{j,k}$
    \item General Identifiability Assumption (GIA): There exists an arm $i^* \in \mathcal{L}$ such that for any $j \in \mathcal{L} \backslash i$ we have $\min_{k \in \mathcal{L}} \Delta_{i^*,k} -  \Delta_{j,k} > 0$
    \item Condorcet Winner (CO): An arm $i \in \mathcal{L}$ is said to be a condorcet winner, if $\forall j \in \mathcal{L} \backslash i$, we have $\Delta_i,j > 0$
\end{itemize} 
These assumptions characterize the structural conditions under which our datasets can be analyzed in the dueling bandit framework. To make this connection explicit, we summarize in Table~\ref{tab:datasets_assumptions} the assumptions that are satisfied by each of the real-world datasets considered in our experiments.
\begin{table}[htb]
\caption{Assumptions followed across real-world datasets.}
\centering
\begin{tabular}{|l|c|c|c|c|}
\hline
Assumption & Text-to-Image & Text-to-Text & Vision & Search \\ \hline \hline
LNM & $\times$ & $\times$ & $\times$ & $\times$ \\ \hline
TO & $\times$ & $\checkmark$ & $\times$ & $\checkmark$ \\ \hline
SST & $\times$ & $\times$ & $\times$ & $\times$ \\ \hline
$\gamma$-RST & $\times$ & $\checkmark$ & $\times$ & $\checkmark$ \\ \hline
MST & $\times$ & $\checkmark$ & $\times$ & $\times$ \\ \hline
WST & $\times$ & $\checkmark$ & $\times$ & $\checkmark$ \\ \hline
STI & NA & $\times$ & NA & $\times$ \\ \hline
GIA & $\times$ & $\times$ & $\times$ & $\times$ \\ \hline
CW & $\checkmark$ & $\checkmark$ & $\checkmark$ & $\checkmark$ \\ \hline
\end{tabular}
\label{tab:datasets_assumptions}
\end{table}

From Table~\ref{tab:datasets_assumptions}, it is evident that each dataset satisfies a distinct subset of assumptions, reflecting their differing structural properties. This observation highlights the inherent diversity among the datasets and the varying nature of their underlying preference relationships. Consequently, we now look into each dataset in detail, presenting its corresponding preference matrix, associated model costs, and representative examples of violated assumptions.

\subsection{Text-To-Image Generation}
The Text-to-Image generation dataset consists of $8$ models evaluated on a variety of image synthesis tasks based on textual prompts. The preference matrix is constructed using pairwise human judgments on the generated outputs, as collected and released by LMSYS. Model costs are obtained from the inference prices reported on the LMSYS leaderboard at the time of data collection. This dataset enables the evaluation of algorithmic performance in a generative visual modality, where subjective human preferences play a significant role in scoring. However, since the LMSYS leaderboard has since been updated and no longer contains this information, we anonymize the models and refer to them as $A$ through $H$ in the subsequent analysis.

\begin{figure}[htbp]
  \centering
  \includegraphics[height=0.65\textwidth, width=0.65\textwidth]{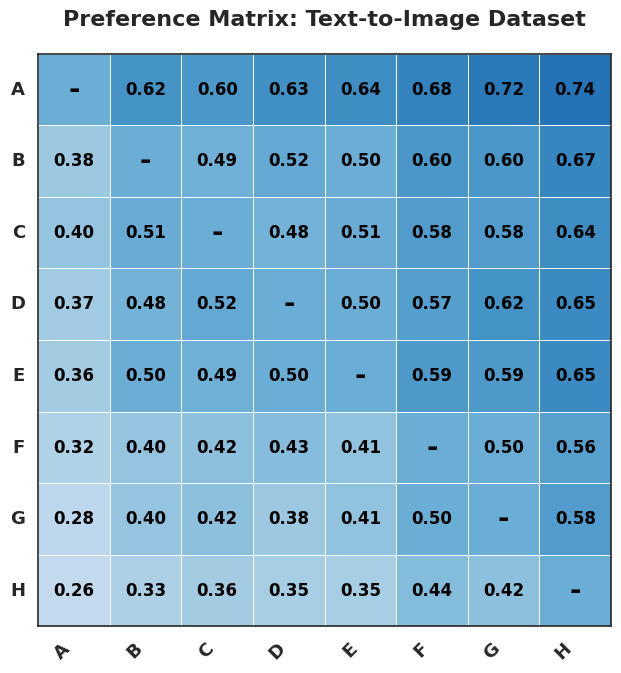}\hfill
  \includegraphics[height=0.65\textwidth, width=0.3\textwidth]{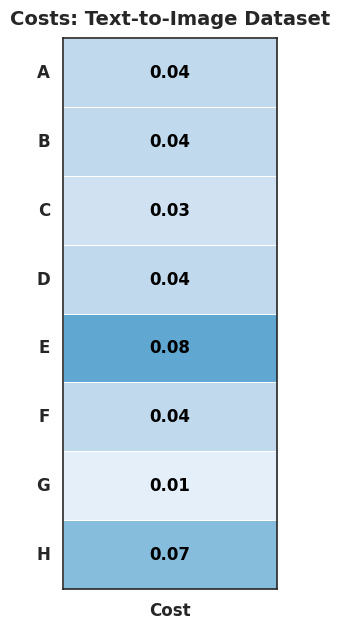}
  \caption{Illustration of the preference matrix and corresponding model costs for the Text-to-Image Generation task. The reported cost represents the amount required to generate one image.}
  \label{fig:text2image_matrix}
\end{figure}

The preference matrix, along with the corresponding cost values for each model, is shown in Figure~\ref{fig:text2image_matrix}. The entries represent the pairwise preferences between models, capturing the probability that one model is favored over another based on collected human comparisons. The costs, shown alongside the matrix, denote the per-query inference prices of the respective models. These values play a central role in analyzing the cost-awareness of the algorithms under study. 

Next, while the dataset satisfies the Condorcet-Winner (CW) assumption, certain other structural assumptions from the dueling bandit literature, such as Total Ordering (TO) and the Low Noise Model (LNM), are not fully preserved. For instance, we observe preference cycles among certain models (e.g., between $B$, $D$, and $C$), which violate transitivity. A detailed set of such violations is provided in Table~\ref{tab:violations_text2image}.

\begin{table}[htb]
\centering
\caption{Assumption-violating examples in Text-to-Image Generation task.}
\begin{tabular}{|c|c|}
\hline
\textbf{Assumption} & \textbf{Violating Example} \\ \hline
Low Noise Model & i=C, j=B\\ \hline
Strong Stochastic Transitivity & i=A, j=D, k=C\\ \hline
$\gamma-$ Relaxed Stochastic Transitivity & i=B, j=D, k=C\\ \hline
Moderate Stochastic Transitivity  & i=B, j=D, k=C\\ \hline
Weak Stochastic Transitivity & i=B, j=D, k=C\\ \hline
\end{tabular}
\label{tab:violations_text2image}
\end{table}

In summary, the Text-to-Image dataset offers a realistic setting where both cost and subjective human preferences influence performance. Although the Condorcet-Winner assumption holds, several other structural assumptions are violated, including transitivity-related properties. These observations highlight the importance of using algorithms that can handle such inconsistencies in real-world settings.

\subsection{Text-To-Text Generation}
The Text-to-Text generation dataset consists of $5$ models evaluated on standard language generation tasks, including instruction following, summarization, and open-ended completion. Similar to the previous setting, pairwise preferences are constructed using human evaluations of model outputs, as provided by LMSYS. The cost of each model is derived from its reported inference price (per million tokens) at the time of data collection.

Compared to the Text-to-Image dataset, this setting involves more structured output and generally exhibits less variability in preferences. As a result, we observe better alignment with several assumptions from the dueling bandit literature, particularly those related to transitivity.

\begin{figure}[htbp]
  \centering
  \includegraphics[height=0.65\textwidth, width=0.6\textwidth]{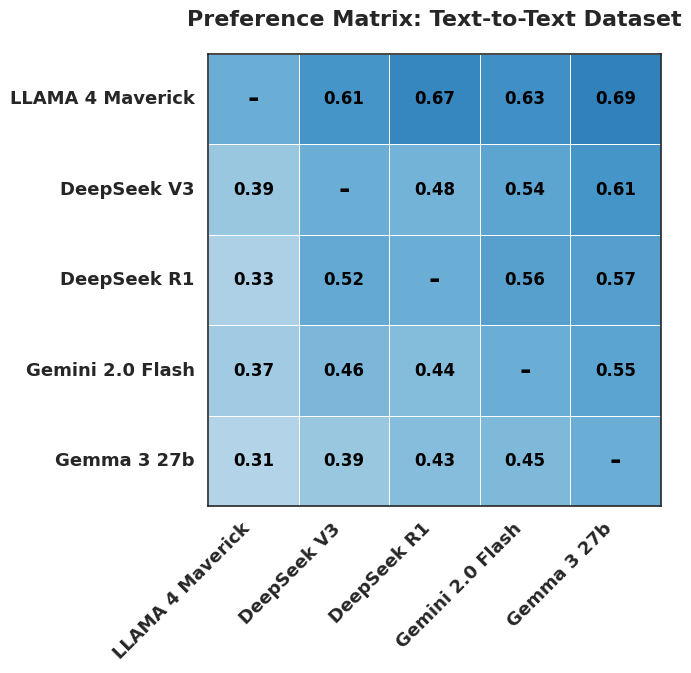}\hfill
  \includegraphics[height=0.65\textwidth, width=0.38\textwidth]{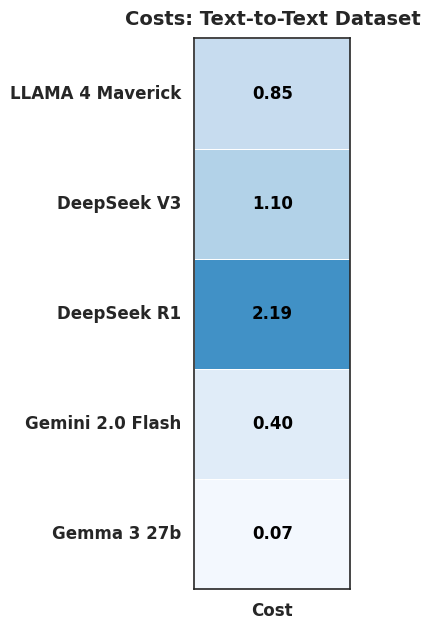}
  \caption{Illustration of the preference matrix and corresponding model costs for the Text-to-Text Generation task. The reported cost represents the amount required to generate one million tokens.}
  \label{fig:text2text_matrix}
\end{figure}

The preference matrix and the corresponding model costs are shown in Figure~\ref{fig:text2text_matrix}. We clearly see that the dataset satisfies the Condorcet-Winner (CW) assumption, with \textbf{LLAMA 4 Maverick} consistently outperforming others across pairwise comparisons. In addition, several assumptions, including \textit{Total Ordering}, $\gamma$-\textit{Relaxed Stochastic Transitivity}, \textit{Moderate Stochastic Transitivity}, and \textit{Weak Stochastic Transitivity} are also satisfied. However, certain assumptions, such as the Low Noise Model and Strong Stochastic Transitivity, are still violated in specific cases. Examples of such violations are summarized in Table~\ref{tab:violations_text2text}.

\begin{table}[htb]
\centering
\caption{Assumption-violating examples in Text-to-Text Generation task.}
\begin{tabular}{|c|c|}
\hline
\textbf{Assumption} & \textbf{Violating Example} \\ \hline
Low Noise Model & i=DeepSeek R1, j = DeepSeek V3\\ \hline
Strong Stochastic Transitivity & i=LLAMA 4 Maverick, j=DeepSeek R1, k=DeepSeek V3\\ \hline
Stochastic Triangle Inequality & i=LLAMA 4 Maverick, j=Gemini 2.0 Flash, k=Gemma 3 27b\\ \hline
\end{tabular}
\label{tab:violations_text2text}
\end{table}

Overall, the Text-to-Text dataset offers a relatively structured and low-noise environment, making it well-suited for evaluating algorithms under conditions where multiple theoretical assumptions are likely to hold.

\subsection{Vision Tasks}

The Vision dataset comprises models evaluated on tasks involving visual understanding, such as image classification and visual question answering. Unlike text-based tasks, outputs in this setting are often more open-ended, making human preference judgments more subjective and less consistent. Consequently, the resulting preference matrix tends to exhibit higher noise and weaker structural regularity compared to the Text-to-Text dataset.

\begin{figure}[htbp]
  \centering
  \includegraphics[height=0.65\textwidth, width=0.6\textwidth]{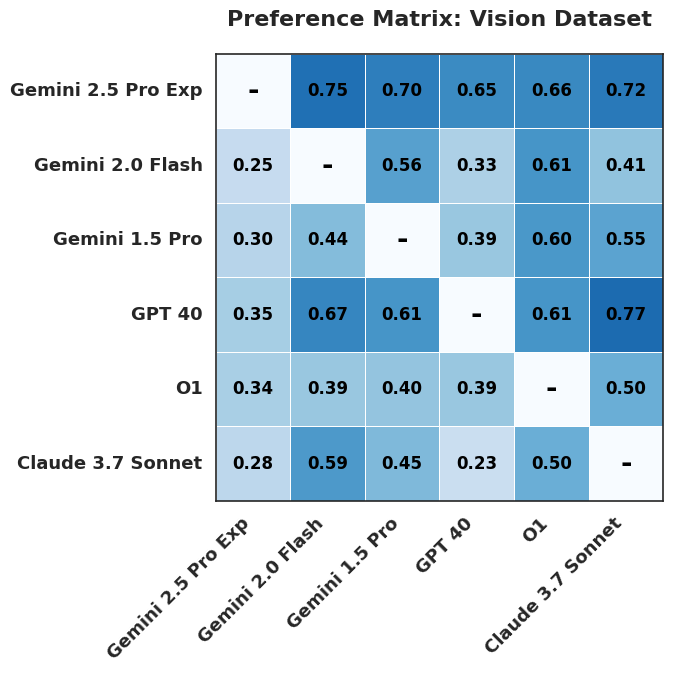}\hfill
  \includegraphics[height=0.65\textwidth, width=0.38\textwidth]{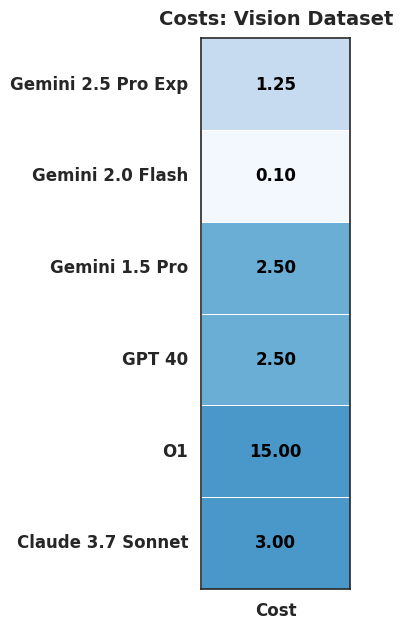}
  \caption{Illustration of the preference matrix and corresponding model costs for the Vision-based task. The reported cost represents the amount required to generate one million tokens.}
  \label{fig:vision_matrix}
\end{figure}

The preference matrix and the corresponding model costs are shown in Figure~\ref{fig:vision_matrix}. As expected, the preferences in this setting are less consistent, and we observe greater uncertainty in pairwise comparisons. While the Condorcet-Winner assumption continues to hold, most of the other structural assumptions are violated. Examples of such violations are presented in Table~\ref{tab:violations_vision}

\begin{table}[htb]
\centering
\caption{Assumption-violating examples in Vision-based task.}
\begin{tabular}{|c|c|}
\hline
\textbf{Assumption} & \textbf{Violating Example} \\ \hline
Low Noise Model & i = Gemini 2.0 Flash j = Gemini 1.5 Pro\\ \hline
Strong Stochastic Transitivity & i=Gemini 2.5 Pro Exp, j=Gemini 2.0 Flash, k=Gemini 1.5 Pro\\ \hline
$\gamma-$ Relaxed Stochastic Transitivity & i=Gemini 2.0 Flash, j=Gemini 1.5 Pro, k=Claude 3.7 Sonnet\\ \hline
Moderate Stochastic Transitivity  & i=Gemini 2.0 Flash, j=Gemini 1.5 Pro, k=Claude 3.7 Sonnet\\ \hline
Weak Stochastic Transitivity & i=Gemini 2.0 Flash, j=Gemini 1.5 Pro, k=Claude 3.7 Sonnet\\ \hline
\end{tabular}
\label{tab:violations_vision}
\end{table}

This setting thus reflects the challenges of working with noisier, less structured preference data, making it useful for testing the robustness of cost-aware algorithms under subjective and inconsistent feedback.

\subsection{Search-Based Tasks}
The Search-based dataset includes models that perform retrieval-augmented generation or open-domain question answering in response to user queries. These tasks emphasize both factual correctness and relevance, often requiring the model to search through the internet to retrieve the relevant information. While preference scores are again derived from human evaluations, the nature of these tasks introduces different sources of variability, such as incomplete or wrong retrieval. However, unlike the Text-to-Image or Vision-based tasks, the scores include a more structured preference matrix. As a result, similar to Text-to-Text datasets, for Search-based tasks, we observe better alignment with various assumptions in the literature

\begin{figure}[htbp]
  \centering
  \includegraphics[height=0.65\textwidth, width=0.63\textwidth]{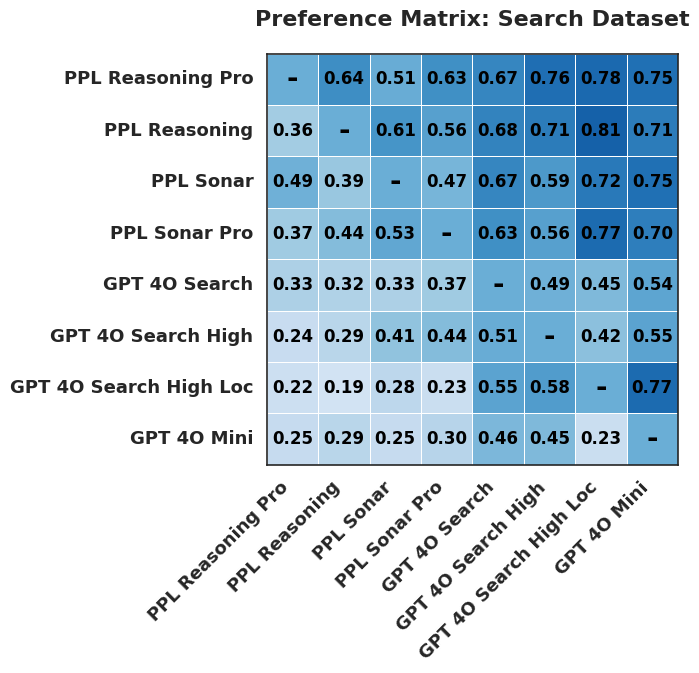}\hfill
  \includegraphics[height=0.65\textwidth, width=0.35\textwidth]{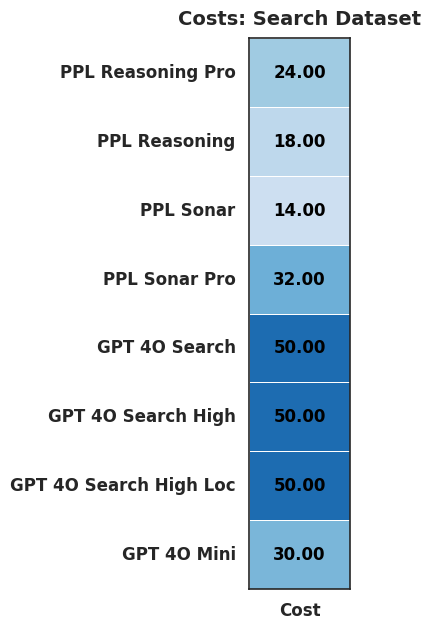}
  \caption{Illustration of the preference matrix and corresponding model costs for the Search-based task. The reported cost represents the amount required to generate one million tokens with high priority.}
  \label{fig:search_matrix}
\end{figure}

The preference matrix and the corresponding model costs for the Search-based dataset are shown in Figure~\ref{fig:search_matrix}. The matrix captures the pairwise win probabilities derived from human comparisons, while the cost values reflect the per-1,000-request pricing in the highest available inference tier (“High” or equivalent), as reported on the official webpages of each model. To enable consistent comparisons across modalities, we normalize these costs under the assumption that 1,000 requests correspond to approximately 1 million input and output tokens. Similar to other datasets, we include only models for which cost information is publicly available.

\begin{table}[htb]
\centering
\caption{Assumption-violating examples in Search-based task.}
\begin{tabular}{|c|c|}
\hline
\textbf{Assumption} & \textbf{Violating Example} \\ \hline
Low Noise Model & i=PPL Sonar Pro, j=PPL Sonar\\ \hline
Strong Stochastic Transitivity & i=PPL Reasoning Pro, j=PPL Reasoning, k=PPL Sonar\\ \hline
Moderate Stochastic Transitivity  & i=PPL Reasoning Pro, j=PPL Reasoning, k=PPL Sonar\\ \hline
Stochastic Triangle Inequality & i=PPL Reasoning Pro, j=PPL Sonar, k=GPT 4O Search High\\ \hline
\end{tabular}
\label{tab:violations_search}
\end{table}

As visible from the preference matrix, the Search-based dataset satisfies the Condorcet-Winner assumption, with \textbf{PPL Reasoning Pro} consistently outperforming the rest in pairwise comparisons. Additionally, several assumptions also hold in this setting. However, similar to the Text-to-Text dataset, assumptions like Strong Stochastic Transitivity, the Low Noise Model, and the Stochastic Triangle Inequality are not universally satisfied. Table~\ref{tab:violations_search} highlights specific examples of such violations.

We end this section by providing a comparative summary of how the assumptions behave across all datasets. Across the four datasets, we observe consistent satisfaction of the Condorcet-Winner assumption, but varying degrees of alignment with other structural assumptions, such as transitivity. These differences are closely tied to the nature of each task. The Text-to-Text and Search-based datasets typically yield more stable and structured preference matrices, owing to their clearer evaluation criteria and constrained outputs. In contrast, the Text-to-Image and Vision datasets tend to introduce greater subjectivity, leading to noisier preferences and more frequent violations of standard assumptions.

This variation in structural properties, combined with differences in model costs across modalities, allows us to evaluate the proposed algorithms in both well-behaved and challenging settings. We next outline the baseline algorithms considered in the experimental evaluation.
\section{Validation on Additional Real-World Datasets}
\label{app: add_val}

In this section, we provide additional empirical support for the assumption that a Condorcet winner exists. 
We consider real-world pairwise-comparison datasets from HuggingFace, including the LMSYS Chatbot Arena dataset and the TigerLab Text-to-Video and Text-to-Image datasets. 
These datasets cover text generation, text-to-video generation, and text-to-image generation tasks, allowing us to validate the Condorcet-winner assumption across different domains.

The purpose of this section is not to perform another cost-aware evaluation, but rather to examine whether the structural assumption used in our theoretical development is consistent with real-world preference data. 
Accordingly, all datasets considered here are used only for validating the Condorcet-winner assumption. 
We do not perform cost-aware experiments on these datasets, since the sampling costs corresponding to all the models considered in these datasets are not publicly available. 
For each dataset, we restrict attention to a subset of models such that every selected model pair has at least one comparison. 
This ensures that the empirical preference matrix is fully specified for the selected set of models.

Table~\ref{tab:additional_validation_assumptions} summarizes which structural assumptions are satisfied by each dataset. 
The main takeaway is that the Condorcet-winner assumption holds across all three datasets, while several stronger assumptions commonly used in the dueling-bandit literature are violated. 
This is important because it suggests that the Condorcet-winner assumption is mild enough to hold in diverse real-world preference datasets, even when more restrictive structural conditions fail.

\begin{table}[htb]
\caption{Assumptions satisfied by the additional real-world datasets considered in our validation. 
The Chatbot Arena dataset is provided by LMSYS, while the Text-to-Video, and Text-to-Image datasets are provided by TigerLab.}
\centering
\begin{tabular}{|l|c|c|c|}
\hline
Assumption & Chatbot Arena & Text to Video & Text to Image  \\ \hline \hline
LNM & $\times$ & $\times$ & $\checkmark$  \\ \hline
TO & $\times$ & $\times$ & $\checkmark$  \\ \hline
SST & $\times$ & $\times$ & $\times$ \\ \hline
$\gamma$-RST & $\times$ & $\times$ & $\checkmark$ \\ \hline
MST & $\times$ & $\times$ & $\times$  \\ \hline
WST & $\times$ & $\times$ & $\checkmark$ \\ \hline
STI & NA & NA & $\times$ \\ \hline
GIA & $\times$ & $\times$ & $\times$  \\ \hline
CW & $\checkmark$ & $\checkmark$ & $\checkmark$  \\ \hline
\end{tabular}
\label{tab:additional_validation_assumptions}
\end{table}

We next discuss each dataset in more detail. 
For each dataset, we present the empirical preference matrix, identify the corresponding Condorcet winner, and provide explicit examples showing which stronger structural assumptions are violated.

\subsection{LMSYS Chatbot Arena}
\label{app:add_val:lmsys}

The LMSYS Chatbot Arena dataset \citep{chiang2024chatbot} consists of pairwise human preference judgments between responses generated by different language models. 
The dataset is publicly available through HuggingFace\footnote{\url{https://huggingface.co/datasets/lmsys/chatbot\_arena\_conversations}}. 
This dataset serves as a text-generation validation case, where preferences are collected from human judgments over model responses.

\begin{figure}[htbp]
  \centering
  \includegraphics[scale=0.55]{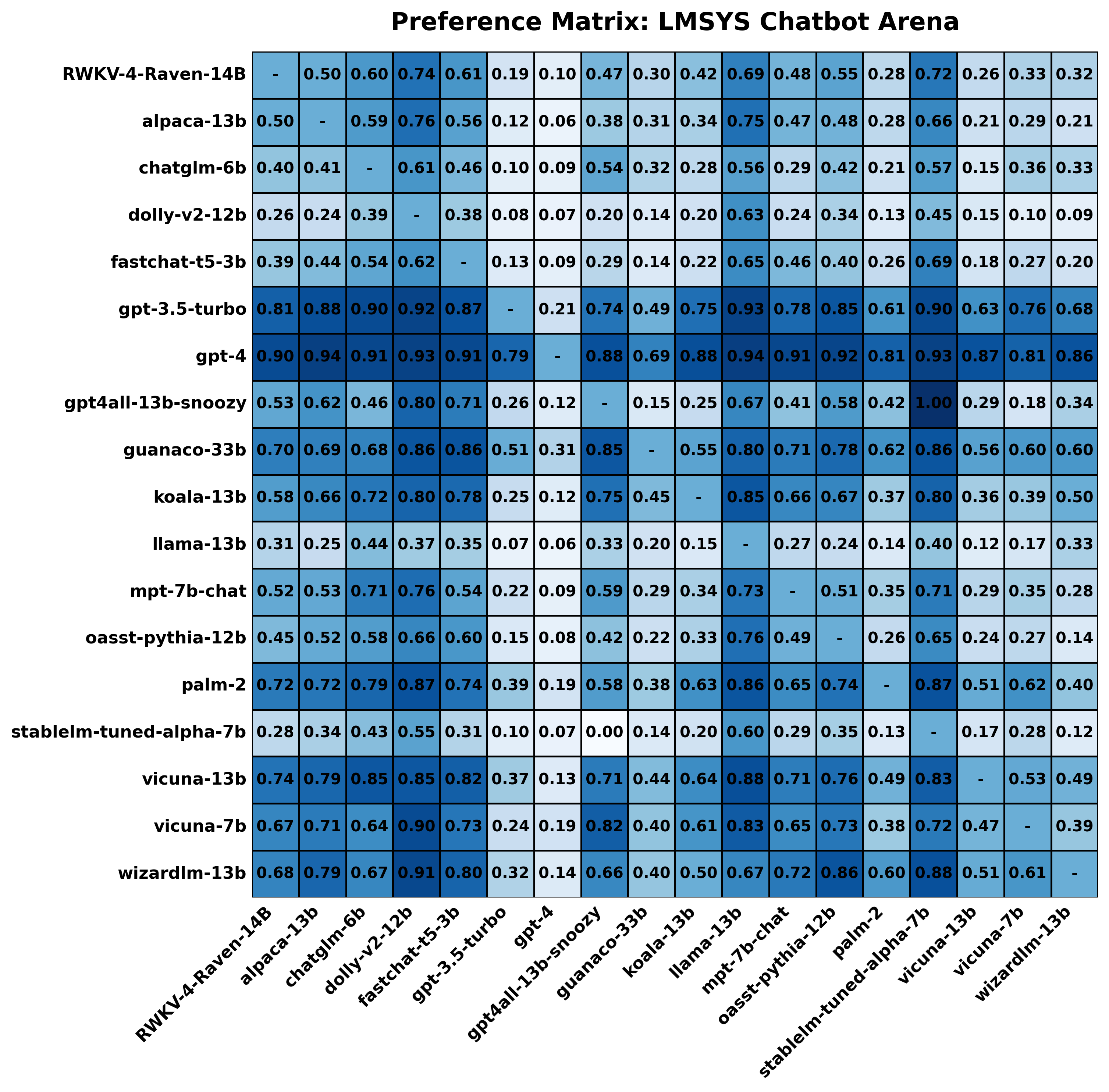}\hfill
  \caption{Illustration of the empirical preference matrix for the LMSYS Chatbot Arena dataset, showing the pairwise preference probabilities among the models.}
  \label{fig:lmsys_pref_matrix}
\end{figure}

Figure~\ref{fig:lmsys_pref_matrix} shows the empirical preference matrix for the selected models. 
From the matrix, we observe that \texttt{GPT 4} is preferred over every other selected model with probability greater than $1/2$, and hence is a Condorcet winner. 
At the same time, the preference matrix does not satisfy stronger structural assumptions such as total ordering. 
Thus, even in this text-generation dataset, the Condorcet-winner assumption holds despite the absence of several more restrictive preference structures.

\begin{table}[htb]
\centering
\caption{Assumption-violating examples in the LMSYS Chatbot Arena dataset.}
\begin{tabular}{|c|c|}
\hline
\textbf{Assumption} & \textbf{Violating Example} \\ \hline
Low Noise Model & i = chatglm-6b, j = gpt4all-13b-snoozy\\ \hline
Strong Stochastic Transitivity & i=RWKV-4-Raven-14B, j=alpaca-13b, k=dolly-v2-12b\\ \hline
$\gamma-$ Relaxed Stochastic Transitivity & i=RWKV-4-Raven-14B, j=chatglm-6b, k=gpt4all-13b-snoozy\\ \hline
Moderate Stochastic Transitivity  & i=RWKV-4-Raven-14B, j=chatglm-6b, k=gpt4all-13b-snoozy\\ \hline
Weak Stochastic Transitivity & i=RWKV-4-Raven-14B, j=chatglm-6b, k=gpt4all-13b-snoozy\\ \hline
\end{tabular}
\label{tab:lmsys_validation_assumptions}
\end{table}

Table~\ref{tab:lmsys_validation_assumptions} provides explicit examples of violated structural assumptions in the LMSYS Chatbot Arena dataset. 
These examples show that the observed preferences are not explained by stronger transitivity or low-noise conditions, while still admitting a Condorcet winner.

\subsection{TigerLab Text-to-Video}
\label{app:add_val:tigerlab_t2v}

The TigerLab Text-to-Video dataset \citep{jiang2024genai} consists of pairwise preference judgments between text-to-video generation models.
The dataset is publicly available through HuggingFace\footnote{\url{https://huggingface.co/datasets/TIGER-Lab/GenAI-Bench/viewer/video_generation}}. 
Compared to LMSYS Chatbot Arena, this dataset provides a validation case from a multimodal generation task, where the outputs are videos rather than text responses.

\begin{figure}[htbp]
  \centering
  \includegraphics[scale=0.5]{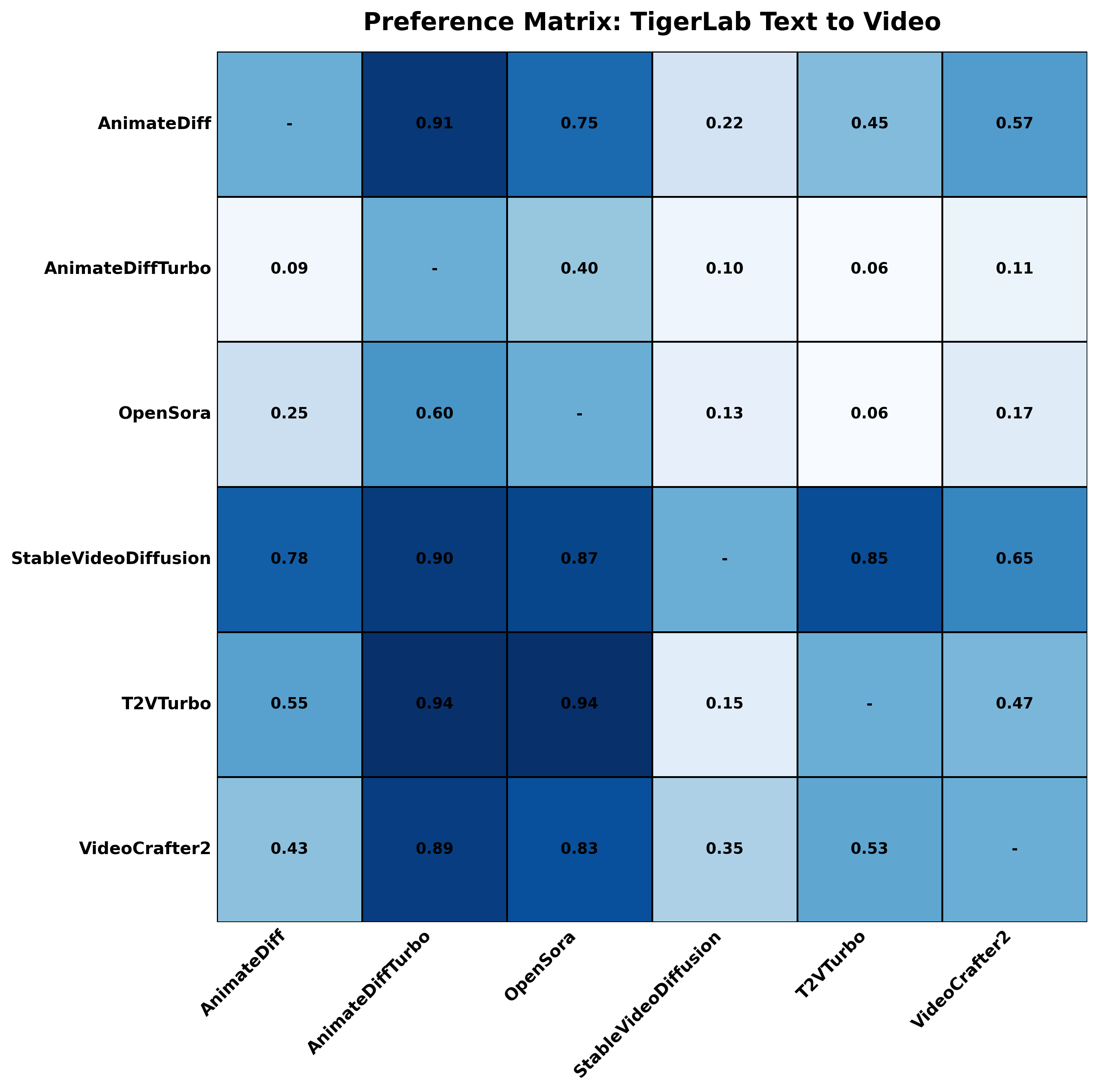}
  \caption{Illustration of the empirical preference matrix for the TigerLab Text-to-Video dataset, showing the pairwise preference probabilities among the models.}
  \label{fig:t2v_pref_matrix}
\end{figure}

Figure~\ref{fig:t2v_pref_matrix} shows the empirical preference matrix for the selected models.
From the matrix, we observe that \texttt{StableVideoDiffusion} is preferred over every other selected model with probability greater than $1/2$, and hence is a Condorcet winner. 
However, similar to the LMSYS Chatbot Arena dataset, the preference matrix does not satisfy a total-order structure or any other assumptions. 
This indicates that the Condorcet-winner assumption continues to hold even in a more complex multimodal preference dataset where stronger ordering assumptions fail.

\begin{table}[htb]
\centering
\caption{Assumption-violating examples in the TigerLab Text-to-Video dataset.}
\begin{tabular}{|c|c|}
\hline
\textbf{Assumption} & \textbf{Violating Example} \\ \hline
Low Noise Model & i = AnimateDiff, j = VideoCrafter2\\ \hline
Strong Stochastic Transitivity & i=AnimateDiff, j=VideoCrafter2, k=OpenSora\\ \hline
$\gamma-$ Relaxed Stochastic Transitivity & i=AnimateDiff, j=VideoCrafter2, k=T2VTurbo\\ \hline
Moderate Stochastic Transitivity  & i=AnimateDiff, j=VideoCrafter2, k=T2VTurbo\\ \hline
Weak Stochastic Transitivity & i=AnimateDiff, j=VideoCrafter2, k=T2VTurbo\\ \hline
\end{tabular}
\label{tab:t2v_validation_assumptions}
\end{table}

Table~\ref{tab:t2v_validation_assumptions} provides explicit examples of violated structural assumptions in the Text-to-Video dataset. 
Together with Figure~\ref{fig:t2v_pref_matrix}, these results show that the empirical preferences admit a Condorcet winner without satisfying stronger transitivity or low-noise assumptions.

\subsection{TigerLab Text-to-Image}
\label{app:add_val:tigerlab_t2i}

The TigerLab Text-to-Image dataset \citep{jiang2024genai} consists of pairwise preference judgments for text-to-image generation models and is publicly available on HuggingFace\footnote{\url{https://huggingface.co/datasets/TIGER-Lab/GenAI-Bench/viewer/image_generation}}. 
This dataset provides another multimodal validation case, now for image generation models.

\begin{figure}[htbp]
  \centering
  \includegraphics[scale=0.5]{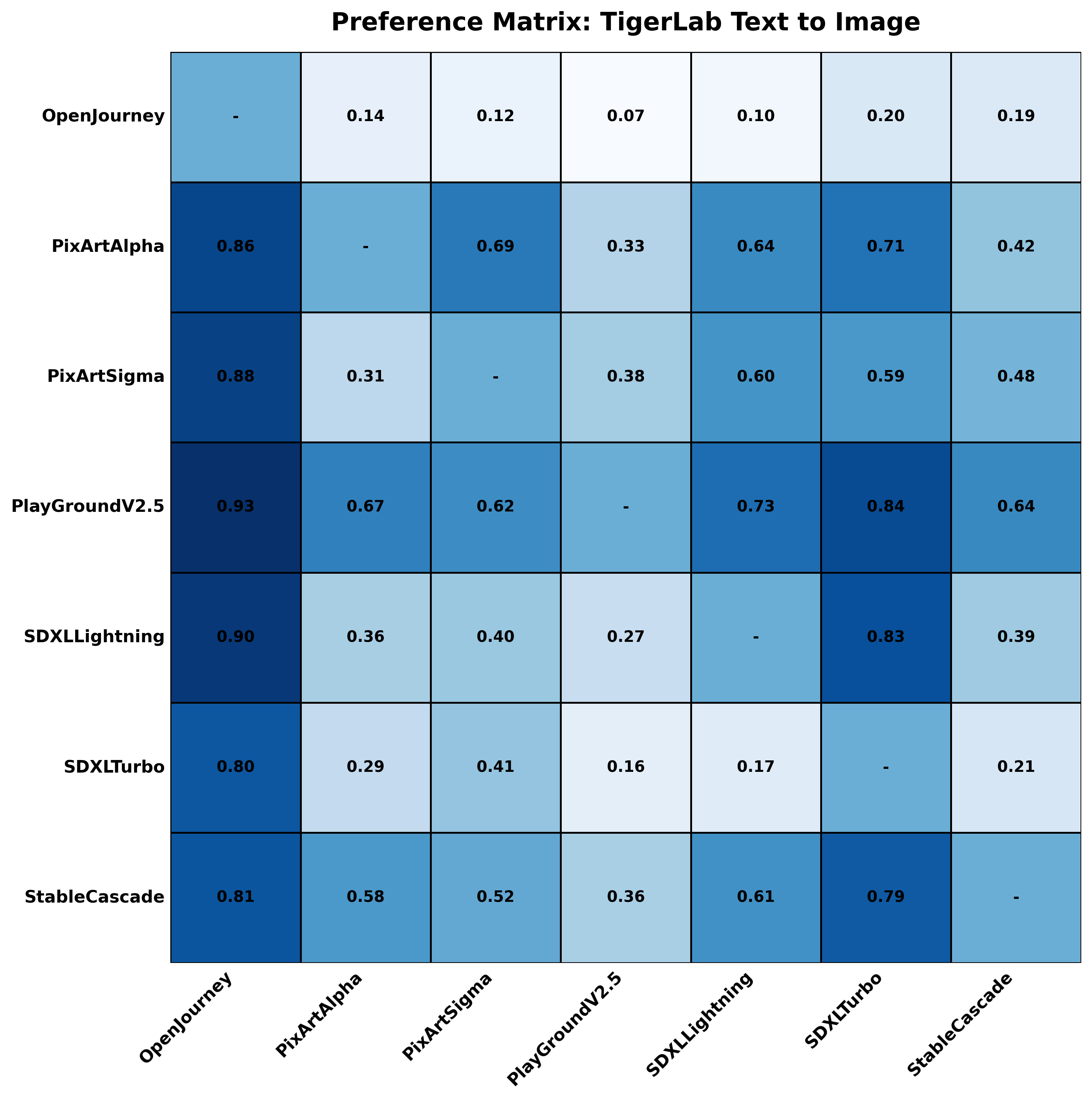}
  \caption{Illustration of the empirical preference matrix for the TigerLab Text-to-Image dataset, showing the pairwise preference probabilities among the models.}
  \label{fig:t2i_pref_matrix}
\end{figure}

Figure~\ref{fig:t2i_pref_matrix} shows the empirical preference matrix for the selected models. 
From the matrix, we observe that \texttt{PlayGroundV2.5} is preferred over every other selected model with probability greater than $1/2$, and hence is a Condorcet winner. 
Unlike the previous two datasets, this dataset also satisfies a total-order structure over the selected models. 
Nevertheless, it still violates other stronger structural assumptions, such as strong stochastic transitivity and the stochastic triangle inequality.

\begin{table}[htb]
\centering
\caption{Assumption-violating examples in the TigerLab Text-to-Image dataset.}
\begin{tabular}{|c|c|}
\hline
\textbf{Assumption} & \textbf{Violating Example} \\ \hline
Strong Stochastic Transitivity & i=PixArtAlpha, j=PixArtSigma, k=OpenJourney\\ \hline
Moderate Stochastic Transitivity  & i=PixArtSigma, j=SDXLLightning, k=SDXLTurbo\\ \hline
Stochastic Triangle Inequality & i=PlayGroundV2.5, j=PixArtSigma, k=SDXLLightning\\ \hline
\end{tabular}
\label{tab:t2i_validation_assumptions}
\end{table}

Table~\ref{tab:t2i_validation_assumptions} provides explicit examples of structural assumptions violated by the Text-to-Image dataset. 
This dataset is therefore complementary to the previous two: although it satisfies total ordering, it still violates other assumptions that are stronger than the existence of a Condorcet winner.

Overall, the validation across these three datasets shows that the Condorcet-winner assumption is consistently satisfied across text-generation, text-to-video generation, and text-to-image generation tasks. 
At the same time, stronger structural assumptions fail in multiple cases. 
This supports our use of the Condorcet-winner assumption as a comparatively mild and empirically plausible condition for real-world pairwise preference data.


\end{document}